\documentclass[conference]{IEEEtran}
\IEEEoverridecommandlockouts
\usepackage{amsmath,amssymb,amsfonts,amsthm}
\usepackage{algorithmic}
\usepackage{graphicx}
\usepackage{textcomp}
\usepackage{xcolor}
\usepackage{booktabs}
\usepackage{cite}

\newtheorem{definition}{Definition}
\newtheorem{theorem}{Theorem}
\newtheorem{lemma}{Lemma}
\newtheorem{proposition}{Proposition}
\newtheorem{corollary}{Corollary}
\newtheorem{remark}{Remark}
\newtheorem{assumption}{Assumption}
\newtheorem{example}{Example}

\begin{document}

\title{What Does the Rank Buy? A Spectral and Distributional Analysis of Low-Rank Adaptation}

\author{\IEEEauthorblockN{Babak Barazandeh}
\IEEEauthorblockA{\textit{Cribl AI Research Lab}}\\
  \small \texttt{bbarazandeh@cribl.io}}

\maketitle
\begin{abstract}
The rank $r$ in LoRA is widely treated as a capacity control: a smaller rank
is assumed to yield a simpler model that generalizes better. We show that,
under hard per-factor norm budgets---the idealization of the weight decay
and norm control used in practice---this intuition breaks down. The reason
is structural: under such budgets, the updates LoRA can reach are exactly
the matrices of rank at most $r$ inside a nuclear-norm ball, and every
complexity and displacement functional we analyze is maximized over this
set by a rank-one update---so the rank cap never binds. The consequences
follow directly. The linear-readout model class we study is identical for
every $r \ge 1$, its Rademacher complexity carries no dependence on $r$,
and the distance the adaptation can move the source distribution obeys a
rank-independent upper bound that we show is sharp. If rank does not
control capacity, where does it act? We identify two places. Statistically,
replacing the per-factor budgets with a joint budget on the product
restores a data-dependent, rank-sensitive complexity bound---though the
gain appears only for well-spread feature distributions, and the worst
case remains rank-free. Spectrally, rank sets the price of adaptation:
canceling the leading singular directions of the pretrained weight
requires both sufficient rank and sufficient budget. We bound the smallest
rank achieving a desired source--target alignment, with upper and lower
bounds that match under two-sided spectral decay. Together, these results
recast rank as governing which updates are reachable and what cancellation
costs---not how much capacity the model has.
\end{abstract}


\section{Introduction}
 \label{sec:intro}
Large-scale pre-trained models have become the backbone of modern
machine learning, but their scale makes full-parameter fine-tuning
computationally prohibitive for most practitioners: the cost of
updating, storing, and serving a separate copy of every weight for
every downstream task grows linearly with model size. This challenge
has motivated a rich body of work on parameter-efficient fine-tuning
(PEFT), which adapts a frozen backbone by training only a small number
of additional parameters~\cite{ding2023peft, lialin2023scaling,
houlsby2019parameter, li2021prefix}. Among PEFT methods, Low-Rank
Adaptation (LoRA)~\cite{hu2022lora} has emerged as the de facto
default: a frozen weight $W_0$ is adapted by a learned update
$\Delta W = BA$ with $\operatorname{rank}(\Delta W)\le r$, matching
full fine-tuning at a small fraction of the trainable-parameter cost.
 
The success of LoRA has spawned a large ecosystem of variants, and it
is striking how many of them revolve, implicitly or explicitly, around
the treatment of the rank. Some methods allocate rank adaptively
across layers or during training~\cite{zhang2023adalora, mao2024dora,
he2025gora}; others recover high-rank updates by aggregating low-rank
ones over time or across modules~\cite{lialin2023relora, jiang2024mora,
zhang2024less}. A complementary line of work redistributes a fixed
low-rank budget \emph{spatially}: MELoRA trains mini-ensembles of
adapters on diagonal blocks~\cite{ren2024melora}, and Localized LoRA generalizes this idea by composing low-rank updates on arbitrary structured blocks of the
weight matrix, achieving lower approximation error at a matched
parameter budget~\cite{barazandeh2025localized}. Its mixture-of-experts
extension equips such block-wise adapters with adaptive routing so
that different regions of the weight matrix specialize to different
operational regimes~\cite{barazandeh2026lmoe}. Further directions
include structured decompositions and initialization
schemes~\cite{shi2024lold, sun2024svfit, bini2025delora, zhao2025lor2c},
parameter sharing and composition~\cite{bishare2025, he2025rasa,
kopiczko2023vera, liu2024vb, huang2023lorahub, ouyang2025klora,
yu2025moka}, mixture-of-experts integrations that preserve world
knowledge across tasks~\cite{dou2023loramoe, yang2024moral},
non-Euclidean geometries~\cite{yang2024hyperbolic}, and multimodal
tuning~\cite{huang2025keeping}; see~\cite{yang2025lowrank} for a survey.
 
Across this entire design space, one justification recurs whenever a
practitioner keeps $r$ small: the \emph{capacity argument}. A tighter
rank bottleneck should mean a smaller hypothesis class, hence a smaller
generalization gap---the same logic that underlies classical low-rank
matrix recovery, where rank acts as an effective dimension~\cite{candes2011tight}.
Rank is ablated in virtually every empirical study, and rank-allocation
methods~\cite{zhang2023adalora, mao2024dora} are motivated precisely by
the premise that rank is the resource whose placement controls the
complexity of the adapted model. Yet this premise is rarely examined
against the constraint set that LoRA, as actually trained and
regularized, imposes. That examination is the subject of this paper.

\subsection{What This Paper Does}

This paper is about one design parameter: the adaptation rank $r$. We ask whether the capacity argument just described survives contact with the constraint set LoRA actually imposes, and we find that it does not.

The reason is visible in one line. We analyze explicit \emph{hard} per-factor Frobenius budgets, $\|B\|_F\le B_B$ and $\|A\|_F\le B_A$---the hard-constraint idealization of the per-factor norm control (weight decay, explicit projection) used in practice; weight decay itself is a soft penalty and is not equivalent to a hard budget without an additional level-set argument. The set of updates such budgets reach is
\begin{equation}
\label{eq:preview}
\mathcal A_r = \big\{\Delta W : \mathrm{rank}(\Delta W)\le r,\ \|\Delta W\|_* \le B_BB_A\big\},
\end{equation}
the rank-$r$ truncation of a \emph{nuclear-norm} ball (Lemma~\ref{lem:factorization}). A nuclear ball is the convex hull of its rank-one elements, so any functional that is convex and maximized on the boundary---in particular the linear Rademacher functional $\Delta W\mapsto\langle\Delta W,S\rangle_F$ and the Frobenius norm $\Delta W\mapsto\|\Delta W\|_F$---attains its maximum over $\mathcal A_r$ at a rank-one point, for every $r\ge1$ (Lemma~\ref{lem:extremal}). The rank cap is therefore never active in either of these functionals. This single mechanism, which we call \emph{rank-one extremality}, drives every result below and explains why the rank collapses out of the statistical analysis and the geometric analysis alike.

Making this precise is not merely a negative exercise. It tells us exactly what would have to change for the rank to matter, it identifies the regime in which a modified analysis does deliver genuine savings, and it relocates the true role of the rank from capacity to \emph{affordability}: $r$ governs how large a norm budget is \emph{required} to cancel a given spectral block of $W_0$, not the maximal support-functional or Wasserstein radius a fixed budget attains in the functionals studied here. (Which \emph{individual} updates are reachable does still depend on $r$: under a fixed budget $\rho>0$ and with $k:=\min(m,d)\ge2$, one has $\mathcal A_1\subsetneq\mathcal A_2\subsetneq\cdots\subsetneq\mathcal A_k$, and $\mathcal A_r=\mathcal A_k$ for $r\ge k$, where the rank cap is vacuous. Our point is that the extremal functionals analyzed below do not see these inclusions.)

\subsection{Three Questions}

We study LoRA-adapted models under a source distribution $D_S$ and a target $D_T$, and organize the paper around:

\begin{itemize}
    \item[\textbf{Q1.}] \emph{Statistical:} Does constraining the update to rank $r$ shrink the effective hypothesis class, and hence the source-domain generalization gap?
    \item[\textbf{Q2.}] \emph{Spectral:} How large must $r$ be, as a function of the singular-value decay of $W_0$, to bring the source--target alignment error below a tolerance $\epsilon$?
    \item[\textbf{Q3.}] \emph{Geometric:} How far, in Wasserstein distance, can a rank-$r$, norm-budgeted adaptation carry the source distribution?
\end{itemize}

\subsection{Contributions}

\begin{itemize}
    \item[\textbf{C1.}] \textbf{(Rank-one extremality.)} We identify $\mathcal A_r$ exactly as a rank-truncated nuclear ball (Lemma~\ref{lem:factorization}) and show that the two functionals governing our analysis are maximized over it at rank one, independently of $r$ (Lemma~\ref{lem:extremal}). This is the engine for C2 and C5.
    \item[\textbf{C2.}] \textbf{(Rank collapse; Q1, negatively.)} The free-readout LoRA class is \emph{the same set of functions} for every $r\ge1$, and we compute its empirical Rademacher complexity exactly, with no $r$ (Proposition~\ref{prop:collapse}). The collapse persists at the level of the matrix-valued layer: the vector-valued complexity of $\{x\mapsto\Delta Wx:\Delta W\in\mathcal A_r\}$ equals $\rho\,\mathbb E\|S\|_2/n$ exactly, again with no $r$ (Proposition~\ref{prop:matrix-collapse}). No covering-number or chaining refinement can distinguish the scalar free-readout classes, which are identical as sets of functions; for the matrix-valued classes---which differ across $r$ as sets whenever $r<\min(m,d)$ and $\mathrm{span}(\mathcal X)=\mathbb R^d$---our claim is only that the vector-valued Rademacher complexity computed here is rank-free.
    \item[\textbf{C3.}] \textbf{(Exactly where rank enters; Q1, positively and conditionally.)} In the scalar linear-readout analysis rank never enters: fixing the readout alone, relaxing the per-factor budget to a joint Frobenius budget alone, and even both changes together all leave the scalar complexity rank-free (Proposition~\ref{prop:necessity}); the two changes are necessary for rank-dependence to be possible, not sufficient. Rank enters exactly at the matrix-valued level under the joint budget, where the complexity is $\rho\,\mathbb E\|S_r\|_F/n$ exactly (Lemma~\ref{lem:matrix-rad}); its operator-norm upper bound can improve on the rank-free Frobenius branch under an effective-rank condition on the empirical covariance (Remark~\ref{rem:when-rank-helps}), its $\sqrt{q_r}$, $q_r:=\min\{r,m,d\}$ ($=r$ in the practical regime), is the sharp worst-case nuclear--Frobenius comparison factor between the two constraint sets (Remark~\ref{rem:sqrt-r-is-the-gap}), and fixed nonlinear readouts inherit rank-sensitive upper bounds from it under the joint budget (Corollary~\ref{cor:rank-dep}). The worst-case simplification, by contrast, is provably rank-free (Lemma~\ref{lem:matrix-rad}(ii)).
    \item[\textbf{C4.}] \textbf{(Spectral threshold rank; Q2, with matching bounds.)} Under an alignment assumption that we formalize and verify in a concrete Gaussian model---uniformly along the entire family of cancellation weights at which it is invoked---the spectral threshold rank $r_{\mathrm{spec}}(\epsilon)$ of the top-block cancellation certificate satisfies upper and lower bounds with the same dependence on $\epsilon$ and on the spectral decay of $W_0$---of order $(1/\epsilon)^{1/(2\alpha)}$ under polynomial decay and $\log(1/\epsilon)$ under geometric decay---whose constants are governed respectively by $\|\Delta\Sigma\|_2$ and a directional alignment constant $c_1$, and which match in order when $c_1\asymp\|\Delta\Sigma\|_2$, in the non-saturated regime $r_{\mathrm{spec}}<\mathrm{rank}(W_0)$ (Theorem~\ref{thm:minrank}, Proposition~\ref{prop:lower}). The budget-feasible minimal certified rank, whenever it exists, coincides with $r_{\mathrm{spec}}$ (Remark~\ref{rem:rcert}). These feed two generalization routes: transport (Theorem~\ref{thm:gen}) and spectral cancellation (Corollary~\ref{cor:spectral-route}).
    \item[\textbf{C5.}] \textbf{(Cone geometry; Q3.)} Whenever some admissible adapted weight is singular---automatic when $m<d$---the divergence-defined transferability cone has ambient-scale radius: infinite on $\mathbb R^d$, and at least half the domain scale under a ball support restriction (Proposition~\ref{prop:trivial-radius}). Its pushforward-realizable restriction obeys a rank-independent \emph{universal} radius upper bound $\sqrt{\lambda_{\max}(\Sigma_S)}\,\rho$: sharp under a mild budget condition, witnessed by a rank-one Gaussian construction (Theorem~\ref{thm:cone}, Proposition~\ref{prop:tight}). Rank enters these radius bounds only through the nuclear budget required to cancel a given spectral block; the reachable-set inclusions $\mathcal A_1\subsetneq\cdots\subsetneq\mathcal A_{\min(m,d)}$ remain strict when $\min(m,d)\ge2$, but the \emph{universal} radius upper bound does not see them---for a fixed instance the exact radius may still vary with $r$, through the growing transferability intersection.
    \item[\textbf{C6.}] \textbf{(Numerical illustrations.)} In
    controlled synthetic single-layer experiments matching the
    theoretical setting, we verify the predictions of the theory:
    the worst-case generalization gap over $\mathcal A_r$ is flat in
    $r$ and attained by a rank-one update under every nominal rank and
    initialization (Lemma~\ref{lem:extremal}); the empirical ERM gap
    exhibits exactly two regimes---growing while the rank cap gates
    the budget, then rank-free with level linear in $\rho$ once the
    nuclear budget binds; the threshold-rank scaling laws of
    Theorem~\ref{thm:minrank} hold to within $10\%$ when the
    directional alignment condition of Proposition~\ref{prop:lower}
    is enforced, and visibly fail without it, showing the condition
    is load-bearing; and the complexity branches of
    Lemma~\ref{lem:matrix-rad} behave as predicted across feature
    spectra (Section~\ref{sec:experiments}).
\end{itemize}

\subsection{Relation to Prior Work}
 
\subsubsection{Theoretical analyses of LoRA}
The expressive power of low-rank adaptation was characterized
by~\cite{zeng2024the}, who quantify the ranks needed for a LoRA-adapted
network to represent a target model; our Lemma~\ref{lem:factorization} can be read as a
budgeted refinement of the same reachability question at a single
layer, with norm constraints added. Recent work argues that LoRA can
match or beat full fine-tuning when the pre-training/downstream
divergence is effectively low-rank~\cite{zindari2026}. Our results are
compatible with, and sharpen, that picture: the benefit cannot come
from capacity reduction on the source domain, because there is none in
the classes we compute; it must come from the spectral-alignment
channel we isolate in Section~\ref{sec:spectral}. On the optimization side, RefLoRA
derives the optimal refactorization of the two LoRA factors and shows
it flattens the loss landscape~\cite{zhang2025reflora}, while
Bernoulli-LoRA provides convergence guarantees for randomized
factor updates~\cite{sokolov2025bernoulli}. These analyses concern the
\emph{trajectory} by which $\Delta W$ is reached; ours concerns the
\emph{set} of updates reachable at all under hard norm budgets, so the
two are complementary---indeed, whether optimization dynamics under
per-factor weight decay implicitly select low-nuclear-norm solutions
is exactly the question our framework poses but does not settle.
 
\subsubsection{Spectrally aware adaptation}
A growing line of methods parameterizes or modulates the update in the
spectral domain of the pretrained weight: SVFit initializes from the
top singular values of $W_0$~\cite{sun2024svfit}, SMoA modulates the
spectrum directly~\cite{liu2026smoa}, and spectral-aware LoRA variants
have proven effective in speaker verification~\cite{li2025spectral}.
Our Section~\ref{sec:spectral} provides a theoretical counterpart for this design
philosophy: the spectral threshold rank $r_{\mathrm{spec}}(\epsilon)$
identifies the top singular directions of $W_0$ as precisely the
locations where rank and budget genuinely bind, with matching upper
and lower bounds under two-sided spectral decay.
 
\subsubsection{Structural and compositional LoRA variants}
Structural variations of LoRA include dynamic rank
distribution~\cite{mao2024dora}, high-rank updating~\cite{lialin2023relora,
jiang2024mora}, granular and localized low-rank
structure~\cite{ren2024melora, barazandeh2025localized},
lower--diagonal--upper decompositions~\cite{shi2024lold}, and
block-wise low-rank experts with adaptive routing~\cite{barazandeh2026lmoe}.
Composition and sharing approaches include Bi-Share
LoRA~\cite{bishare2025}, Rank-Sharing LoRA~\cite{he2025rasa},
LoRAHub~\cite{huang2023lorahub}, and K-LoRA~\cite{ouyang2025klora};
architecture-specific variants integrate Mixture-of-Experts
designs~\cite{dou2023loramoe, yang2024moral}, mini-ensembles~\cite{ren2024melora},
multimodal contexts~\cite{huang2025keeping}, and hyperbolic
geometry~\cite{yang2024hyperbolic}. We emphasize that our analysis is
conducted for the standard global parameterization $\Delta W = BA$;
whether block-wise parameterizations such as~\cite{barazandeh2025localized,
barazandeh2026lmoe} escape the rank-one extremality mechanism is an
interesting open question, since a budget distributed across blocks
changes the geometry of the reachable set from a single nuclear ball
to a product of smaller ones.
 
\subsubsection{Domain adaptation and technical tools}
On the theory side, our discrepancy framework builds on classical
domain-adaptation theory~\cite{ben2006analysis} and its
optimal-transport variants~\cite{redko2017theoretical, shen2018wasserstein},
our concentration on matrix Bernstein inequalities~\cite{tropp2012user},
and our contraction arguments on~\cite{ledoux1991probability, maurer2016vector}.
 
\subsection{Organization}

Section~\ref{sec:framework} sets up notation, proves the rank-one extremality principle, and defines the discrepancy and the transferability cone. Section~\ref{sec:complexity} answers Q1. Section~\ref{sec:spectral} answers Q2. Section~\ref{sec:cone} answers Q3. Section~\ref{sec:experiments} reports controlled numerical
illustrations of the main results. Appendix~\ref{app:aux} collects auxiliary lemmas, the Gaussian instantiation of our alignment assumption, and the tightness witness; Appendix~\ref{app:remarks} collects secondary results and extended remarks; Appendix~\ref{app:proofs} contains all proofs.

\section{Setup and the Rank-One Extremality Principle}
\label{sec:framework}

\subsection{Background: The LoRA Parameterization}

Consider a pre-trained layer with weight $W_0\in\mathbb R^{m\times d}$,
mapping features $x\in\mathbb R^d$ to outputs $W_0x\in\mathbb R^m$. Full
fine-tuning adapts the layer by replacing $W_0$ with $W_0+\Delta W$ for
an unconstrained update $\Delta W\in\mathbb R^{m\times d}$, at the cost
of training, storing, and serving $md$ parameters per layer per task.
LoRA~\cite{hu2022lora} replaces the unconstrained update with the
factored ansatz
\begin{equation}
\label{eq:lora-ansatz}
\Delta W = BA,\qquad B\in\mathbb R^{m\times r},\quad A\in\mathbb R^{r\times d},
\end{equation}
for a small integer rank $r\ll\min(m,d)$, so that the adapted forward
pass is $x\mapsto W_0x+B(Ax)$. Only $B$ and $A$ are trained---$W_0$
stays frozen---which cuts the trainable parameters from $md$ to
$r(m+d)$ and, since $BA$ can be merged into $W_0$ after training, adds
no inference latency. The construction enforces
$\operatorname{rank}(\Delta W)\le r$ by design; its empirical
justification is that the updates learned during fine-tuning appear to
have low intrinsic rank, so little is lost by imposing the bottleneck.
In practice the factors are additionally norm-controlled, through
weight decay on $B$ and $A$ or explicit projection, and implementations
scale the update by $\alpha_{\mathrm{LoRA}}/r$ (see
Remark~\ref{rem:alpha-scaling}); the hard per-factor budgets analyzed
in this paper are the constraint-set idealization of exactly this norm
control. The mathematical question the parameterization raises is what
the pair (rank cap, norm budget) actually restricts: which updates are
reachable, and how complex the resulting model class is.
\subsection{Notation and Standing Conventions}

Let $\mathcal X\subseteq\mathbb R^d$ be the feature space at the input of the adapted layer. We write $D_S,D_T\in\mathcal P(\mathcal X)$ for the source and target feature distributions, $\Sigma_S=\mathbb E_{D_S}[xx^T]$ and $\Sigma_T=\mathbb E_{D_T}[xx^T]$ for their (uncentered) second-moment matrices, and $\Delta\Sigma=\Sigma_S-\Sigma_T$. The pre-trained weight is $W_0\in\mathbb R^{m\times d}$ and a LoRA update $\Delta W=BA$ is as in \eqref{eq:lora-ansatz}, for an integer rank $r\ge1$, so $\mathrm{rank}(\Delta W)\le r$ automatically. Per-factor budgets are $\|B\|_F\le B_B$ and $\|A\|_F\le B_A$, and we abbreviate their product as $\rho:=B_BB_A$; we assume $B_B,B_A>0$ throughout (if either vanishes, $\mathcal A_r=\{0\}$: the fixed-budget complexity and transport bounds reduce accordingly, while the cancellation results simply have their explicitly stated budget conditions fail). For a matrix $M$ we write $\|M\|_2$, $\|M\|_F$, $\|M\|_*$ for the operator, Frobenius, and nuclear norms, $s_1(M)\ge s_2(M)\ge\cdots$ for its singular values, with the convention $s_i(M)=0$ for $i>\mathrm{rank}(M)$ (we reserve $\sigma$ for Rademacher variables), and $M_k$ for its best rank-$k$ approximation.

We use two Rademacher complexities, both \emph{empirical} (conditional on a fixed sample $x_1,\dots,x_n$). For a class $\mathcal F$ of scalar functions and i.i.d.\ signs $\sigma_i\in\{\pm1\}$,
\begin{equation}
\mathfrak R_n(\mathcal F) := \frac1n\,\mathbb E_\sigma\Big[\sup_{f\in\mathcal F}\sum_{i=1}^n\sigma_if(x_i)\Big];
\end{equation}
for a class $F$ of $\mathbb R^m$-valued functions and i.i.d.\ Rademacher \emph{vectors} $\sigma_i\in\{\pm1\}^m$ (each coordinate an independent sign),
\begin{equation}
\mathfrak R_n^{\mathrm{vec}}(F) := \frac1n\,\mathbb E_\sigma\Big[\sup_{f\in F}\sum_{i=1}^n\langle\sigma_i,f(x_i)\rangle\Big].
\end{equation}
Both are conditional on the sample; we write $\bar{\mathfrak R}_n(\cdot):=\mathbb E_{x_{1:n}}\big[\mathfrak R_n(\cdot)\big]$ for the corresponding \emph{expected} (sample-averaged) complexity. All complexity bounds we prove are uniform over samples satisfying the boundedness assumption below, so they apply verbatim to the expected complexities $\bar{\mathfrak R}_n$ used in Theorem~\ref{thm:gen} and Corollary~\ref{cor:spectral-route}.

\begin{assumption}[Bounded Features]
\label{ass:bddfeat}
$\|x\|_2\le X_{\max}$ almost surely under $D_S$.
\end{assumption}

Assumption~\ref{ass:bddfeat} is invoked \emph{only} where stated (Sections~\ref{sec:complexity} and~\ref{sec:spectral}); it is deliberately not a global standing hypothesis, since several of our sharpness witnesses are Gaussian and would otherwise be excluded. Where we use Gaussian sources we say so explicitly and do not simultaneously invoke Assumption~\ref{ass:bddfeat}.

The central object is the reachable update set,
\begin{equation}
\label{eq:Ar}
\begin{aligned}
\mathcal A_r
&:= \Bigl\{BA:\ B\in\mathbb R^{m\times r},\
A\in\mathbb R^{r\times d},\\
&\qquad
\|B\|_F\le B_B,\quad
\|A\|_F\le B_A
\Bigr\}.
\end{aligned}
\end{equation}

The practical $\alpha_{\mathrm{LoRA}}/r$ scaling used in implementations is accounted for in Remark~\ref{rem:alpha-scaling} (Appendix~\ref{app:remarks}).

\subsection{The Rank-One Extremality Principle}

Everything in this paper follows from two facts about $\mathcal A_r$. The first identifies it; the second says the rank cap is inactive.

\begin{lemma}[Exact Description of the Reachable Set]
\label{lem:factorization}
For every $r\ge1$,
\begin{equation}
\mathcal A_r = \big\{\Delta W\in\mathbb R^{m\times d}\ :\ \mathrm{rank}(\Delta W)\le r,\ \ \|\Delta W\|_*\le\rho\big\}.
\end{equation}
That is, per-factor Frobenius budgets are exactly a \emph{nuclear-norm} budget of radius $\rho=B_BB_A$ on the product, together with the rank cap.
\end{lemma}

\begin{lemma}[Rank-One Extremality]
\label{lem:extremal}
Fix any $M\in\mathbb R^{m\times d}$ and any $r\ge1$. Then
\begin{align}
\sup_{\Delta W\in\mathcal A_r}\ \langle\Delta W,M\rangle_F &= \rho\,\|M\|_2,\label{eq:extremal-lin}\\
\sup_{\Delta W\in\mathcal A_r}\ \|\Delta W\|_F &= \rho.\label{eq:extremal-frob}
\end{align}
Both suprema are attained at rank-one points of $\mathcal A_1\subseteq\mathcal A_r$, and neither depends on $r$.
\end{lemma}

\begin{remark}[Scope of the Extremality Principle]
\label{rem:engine}
Statistical complexity is controlled by suprema of the linear functional \eqref{eq:extremal-lin} against a Rademacher signal matrix; Wasserstein displacement under a linear pushforward is controlled by \eqref{eq:extremal-frob}. Lemma~\ref{lem:extremal} therefore forecloses rank-dependence in both settings simultaneously, and does so for structural reasons---the extremizer of a linear functional over the convex hull of rank-one matrices is rank one---rather than through any looseness of a proof technique. Sections~\ref{sec:complexity} and~\ref{sec:cone} are, in this sense, two readings of the same lemma.
\end{remark}

\subsection{Discrepancy, Cone, and Risk}

\begin{definition}[Score Discrepancy]
\label{def:disc}
Let $\mathcal H$ be a class of real-valued measurable score functions on $\mathcal X$. The score discrepancy between $D_S$ and $D_T$ relative to $\mathcal H$ is
\begin{equation}
\begin{split}
\tilde{d}_{\mathcal H}(D_S, D_T) = \sup_{h, h' \in \mathcal{H}} \Big| & \mathbb{E}_{x \sim D_S}\big[|h(x) - h'(x)|\big] \\
& - \mathbb{E}_{x \sim D_T}\big[|h(x) - h'(x)|\big] \Big|.
\end{split}
\end{equation}
\end{definition}

Remark~\ref{rem:disc-position} (Appendix~\ref{app:remarks}) situates $\tilde d_{\mathcal H}$ among existing divergences; every result below is stated and proved for $\tilde d_{\mathcal H}$ natively.

\begin{definition}[Wasserstein Distances]
\label{def:wass}
For $D_S,D_T\in\mathcal P(\mathcal X)$ and $\Pi(D_S,D_T)$ the set of couplings, define $W_1$ whenever both distributions have finite first moments and $W_2$ whenever both have finite second moments:
\begin{align}
W_1(D_S, D_T) &= \sup_{\|f\|_L \le 1} \left| \mathbb{E}_{D_S}[f(x)] - \mathbb{E}_{D_T}[f(x)] \right|,\\
W_2(D_S,D_T) &= \Big(\inf_{\pi\in\Pi(D_S,D_T)} \mathbb E_{(x,y)\sim\pi}\|x-y\|_2^2\Big)^{1/2},
\end{align}
the first identity being Kantorovich--Rubinstein duality.
\end{definition}

Since adaptation changes the layer weight, the natural hypothesis class against which to measure discrepancy is the one induced by a \emph{fixed adapted weight} with variable readouts.

\begin{definition}[Induced Readout Class]
\label{def:induced}
Fix, for each $W\in\mathbb R^{m\times d}$, a set $\Psi_W$ of measurable readouts $\psi:\mathbb R^m\to\mathbb R$; when a single set is used for all weights we write $\Psi_W\equiv\Psi$. The induced class is $\mathcal H_W := \{x\mapsto\psi(Wx) : \psi\in\Psi_W\}$.
\end{definition}

\begin{definition}[LoRA-Transferability Cone]
\label{def:cone}
Given $W_0$, a source $D_S$, a rank $r$, budgets $(B_B,B_A)$, and a tolerance $\epsilon\ge0$,
\begin{equation}
\begin{split}
\mathcal{C}_r(W_0, D_S;\epsilon) = \Big\{ D_T \in \mathcal{P}(\mathcal{X}) : & \\
\exists\, \Delta W \in\mathcal A_r, \ & \tilde{d}_{\mathcal H_{W_0+\Delta W}}(D_S, D_T) \le \epsilon \Big\}.
\end{split}
\end{equation}
\end{definition}

The norm budgets are part of the cone's definition: without them no radius statement is well-posed. As Proposition~\ref{prop:trivial-radius} shows, even with them the divergence-defined cone has an ambient-scale radius whenever some admissible adapted weight is singular, which is what motivates the pushforward restriction of Section~\ref{sec:cone}.

\paragraph*{Risk model}
We adopt the deterministic labeling-function model of \cite{ben2006analysis}: each domain $D\in\{D_S,D_T\}$ carries a measurable $f_D:\mathcal X\to\mathbb R$, and the risk of a score hypothesis $h$ is $\epsilon_D(h) := \mathbb E_{x\sim D}|h(x)-f_D(x)|$. Given a sample $x_1,\dots,x_{n_S}\sim D_S$, the empirical source risk is $\hat\epsilon_S(h):=\tfrac1{n_S}\sum_{i=1}^{n_S}|h(x_i)-f_S(x_i)|$.

\begin{assumption}[Bounded Loss]
\label{ass:bounded}
There is $b<\infty$ with $|h(x)-f_D(x)|\le b$ for all $h$ in the class under consideration, all $x\in\mathcal X$, and $D\in\{D_S,D_T\}$.
\end{assumption}

\section{Q1: Statistical Complexity---Collapse, and What It Would Take to Avoid It}
\label{sec:complexity}

We first record the baseline adaptation bound, which localizes where rank could possibly enter. We then show it does not enter the classes we compute (Propositions~\ref{prop:collapse} and~\ref{prop:matrix-collapse}), show that no scalar single-layer linear-readout formalization considered here exhibits rank-dependence---each natural modification alone, and even both together, leaves the scalar complexity rank-free, with or without the frozen weight (Proposition~\ref{prop:necessity}, Remark~\ref{rem:frozen-scalar})---and locate exactly where rank does enter: the matrix-valued complexity under the joint budget (Lemma~\ref{lem:matrix-rad}, Corollary~\ref{cor:rank-dep}).

\subsection{The Baseline Adaptation Bound}

\begin{theorem}[Adaptation Bound for the Score Discrepancy]
\label{thm:da}
Let $\mathcal H$ be any class of real-valued score functions and let $\lambda^* := \inf_{h'\in\mathcal H}[\epsilon_S(h')+\epsilon_T(h')]$ be the ideal joint risk (a value, not a minimizer). Then for every $h\in\mathcal H$,
\begin{equation}
\epsilon_T(h) \le \epsilon_S(h) + \tilde{d}_{\mathcal H}(D_S, D_T) + \lambda^*.
\end{equation}
\end{theorem}

Theorem~\ref{thm:da} splits target risk into source performance, a distributional discrepancy, and an ideal joint risk. The first term is controlled by Rademacher complexity, which is the subject of this section; the second by spectral alignment and transport, which is Section~\ref{sec:spectral}. The question throughout is where $r$ enters.

\subsection{Collapse at the Scalar Level}

The most direct formalization of a LoRA-adapted scalar score jointly optimizes a bounded readout direction $\mathbf w$ and the update.

\begin{proposition}[Collapse of the Free-Readout Class]
\label{prop:collapse}
Let $\mathcal F_r$ be the class
\begin{equation}
h_{\mathbf w,\Delta W}(x) = \mathbf{w}^T (W_0 + \Delta W) x,\qquad \|\mathbf{w}\|_2 \le 1,\ \Delta W\in\mathcal A_r.
\end{equation}
Then, writing $v:=\sum_{i=1}^n\sigma_ix_i$:
\begin{enumerate}
\item[(i)] \emph{(Set identity.)} $\mathcal F_r=\mathcal F_1$ for every $r\ge1$, as sets of functions.
\item[(ii)] \emph{(Exact complexity.)} $\displaystyle\mathfrak R_n(\mathcal F_r)=\frac1n\,\mathbb E_\sigma\big[\|W_0v\|_2+\rho\|v\|_2\big]$, with no $r$-dependence, and under Assumption~\ref{ass:bddfeat}
\begin{equation}
\label{eq:collapsed}
\mathfrak{R}_n(\mathcal{F}_r) \le \frac{X_{\max}}{\sqrt{n}}\big(\|W_0\|_2 + \rho\big).
\end{equation}
\end{enumerate}
\end{proposition}

Remark~\ref{rem:collapse-mechanism} (Appendix~\ref{app:remarks}) details the rank-one surrogate mechanism behind part (i) and its consequences for covering-number and chaining refinements.

\subsection{Collapse Persists at the Matrix Level}

A natural response is that the collapse is an artifact of the scalar readout, and that bounding the matrix-valued layer directly will restore $r$. It does not.

\begin{proposition}[Collapse of the Adaptation Layer]
\label{prop:matrix-collapse}
Let $S:=\sum_{i=1}^n\sigma_ix_i^T$ with $\sigma_i\in\{\pm1\}^m$ i.i.d.\ Rademacher vectors. Then for every $r\ge1$,
\begin{equation}
\mathfrak R_n^{\mathrm{vec}}\big(\{x\mapsto\Delta Wx:\Delta W\in\mathcal A_r\}\big) = \frac{\rho}{n}\,\mathbb E_\sigma\|S\|_2,
\end{equation}
exactly, with no dependence on $r$.
\end{proposition}

So under LoRA's own parameterization the rank is invisible to the exact complexity functionals computed here, at both levels. (The underlying matrix classes $\mathcal A_r$ differ across $r$ for $r<\min(m,d)$; the induced sets of vector-valued maps $\{x\mapsto\Delta Wx:\Delta W\in\mathcal A_r\}$ then differ whenever $\mathrm{span}(\mathcal X)=\mathbb R^d$, so that distinct matrices induce distinct maps. What Proposition~\ref{prop:matrix-collapse} shows is that their vector-valued Rademacher complexity does not differ. Whether a \emph{fixed nonlinear} readout composed with these sets can exhibit rank-dependence is not excluded by our upper bounds, which are rank-free but are bounds rather than identities at that level.) This is the paper's core negative finding, and it is worth being precise about what must be given up to escape it in the scalar analysis.

\begin{definition}[Joint-Budget Relaxation]
\label{def:Brho}
For $r\ge1$ let
\begin{equation}
\mathcal B_{r,\rho} := \big\{\Delta W : \mathrm{rank}(\Delta W)\le r,\ \|\Delta W\|_F\le\rho\big\}.
\end{equation}
\end{definition}

Write $q_r:=\min\{r,m,d\}$, the maximal rank available in $\mathcal B_{r,\rho}$. By Lemma~\ref{lem:factorization} and the chain $\|M\|_F\le\|M\|_*\le\sqrt{\mathrm{rank}(M)}\,\|M\|_F\le\sqrt{q_r}\,\|M\|_F$, valid for every $M\in\mathcal B_{r,\rho}$,
\begin{equation}
\label{eq:sandwich}
\mathcal A_r\ \subseteq\ \mathcal B_{r,\rho}\ \subseteq\ \mathcal A_r^{(\sqrt{q_r}\,\rho)},
\end{equation}
where $\mathcal A_r^{(\rho')}:=\{\Delta W:\mathrm{rank}(\Delta W)\le r,\ \|\Delta W\|_*\le\rho'\}$, the reachable set under any per-factor budgets of product $\rho'$ (Lemma~\ref{lem:factorization}); the inclusions are strict for $2\le r\le\min(m,d)$, while for $\min(m,d)=1$ one has $q_r=1$ and all three sets coincide, and $\mathcal B_{1,\rho}=\mathcal A_1$. In the practically relevant regime $r\le\min(m,d)$, $q_r=r$. Operationally, $\mathcal B_{r,\rho}$ is what one obtains by regularizing the \emph{product} $\|BA\|_F$ rather than the factors separately.

One might hope to restore rank-dependence within the scalar analysis by
fixing the readout, or by relaxing the per-factor budget to the joint
budget $\mathcal B_{r,\rho}$. Neither modification alone suffices, and
even both together leave the scalar complexity rank-free, with or
without the frozen weight (Proposition~\ref{prop:necessity} and
Remark~\ref{rem:frozen-scalar}, Appendix~\ref{app:remarks}). The reason
is transparent: a scalar linear readout---free or fixed---makes the
effective signal matrix $\mathbf wv^T$ rank one, and a rank cap
$r\ge1$ cannot bind against a rank-one signal. Rank can only help when
the signal matrix has more than $r$ significant directions, which
requires the matrix-valued setting we treat next.

\subsection{Rank-Sensitivity Under the Joint Budget}

\begin{lemma}[Matrix-Valued Rademacher Complexity Under $\mathcal B_{r,\rho}$]
\label{lem:matrix-rad}
Suppose Assumption~\ref{ass:bddfeat} holds. Let $S:=\sum_{i=1}^n\sigma_ix_i^T$ with $\sigma_i\in\{\pm1\}^m$ i.i.d.\ Rademacher vectors, and set
\begin{equation}
\mathfrak{R}_n^{\mathrm{mat}} := \frac{1}{n}\,\mathbb{E}_{\sigma}\Big[\sup_{\Delta W\in\mathcal B_{r,\rho}} \langle \Delta W,\ S\rangle_F\Big].
\end{equation}
Let $G_n:=\sum_{i=1}^nx_ix_i^T$ and, when $G_n\ne0$ (i.e.\ some $x_i\ne0$), $r_{\mathrm{eff}}(G_n):=\mathrm{tr}(G_n)/\|G_n\|_2\in[1,\min(n,d)]$. Then:
\begin{enumerate}
\item[(i)] \emph{(Exact form and data-dependent bound.)} $\mathfrak R_n^{\mathrm{mat}}=\tfrac\rho n\mathbb E\|S_r\|_F$, and

\begin{equation}
\label{eq:datadep}
\begin{aligned}
\mathfrak{R}_n^{\mathrm{mat}}
&\le \frac{\rho}{n}\,
\min\Bigl\{\sqrt{q_r}\,\mathbb E\|S\|_2,\ \mathbb E\|S\|_F\Bigr\},\\
&\qquad q_r=\min\{r,m,d\}\ \text{as in \eqref{eq:sandwich}},
\end{aligned}
\end{equation}
where, with $v_n := \max\big(\mathrm{tr}(G_n),\ m\|G_n\|_2\big)$,
\begin{align}
\mathbb E\|S\|_F &\le \sqrt{m\,\mathrm{tr}(G_n)},\label{eq:frob-branch}\\
\mathbb E\|S\|_2 &\le \sqrt{2\,v_n\log(m+d)} + \tfrac{\sqrt m\,X_{\max}}{3}\log(m+d).\label{eq:op-branch}
\end{align}
\item[(ii)] \emph{(The worst case is rank-free.)} If $n\ge\log(m+d)$ then $\mathbb E\|S\|_2\le2X_{\max}\sqrt{mn\log(m+d)}$ and $\mathbb E\|S\|_F\le X_{\max}\sqrt{mn}$, hence
\begin{equation}
\label{eq:lem2-fixed}
\begin{aligned}
\mathfrak{R}_n^{\mathrm{mat}}
&\le \rho X_{\max}
 \min\!\left\{
 2\sqrt{\frac{m q_r\log(m+d)}{n}},
 \sqrt{\frac{m}{n}}
 \right\} \\
&\le \rho X_{\max}\sqrt{\frac{m}{n}} .
\end{aligned}
\end{equation}
where the minimum is always attained by the Frobenius branch because $4q_r\log(m+d)\ge4\log2>1$ for every $q_r\ge1$ and $m+d\ge2$: under worst-case features the operator-norm branch never beats the Frobenius branch. The worst-case simplification is therefore rank-free; all rank-sensitivity resides in the exact form and the data-dependent bound of (i).
\end{enumerate}
\end{lemma}

Remarks~\ref{rem:sqrt-r-is-the-gap}, \ref{rem:when-rank-helps}, and~\ref{rem:sqrtm} (Appendix~\ref{app:remarks}) interpret the $\sqrt{q_r}$ factor as the worst-case nuclear--Frobenius gap, identify the effective-rank regime in which the operator-norm branch can improve, and account for the $\sqrt m$ factor.

\begin{definition}[Fixed-Readout LoRA Class]
\label{def:fixed-readout}
Given a \emph{fixed} $L_\psi$-Lipschitz readout $\psi:\mathbb R^m\to\mathbb R$ and a constraint set $\mathcal D_r\in\{\mathcal A_r,\ \mathcal B_{r,\rho}\}$,
\begin{equation}
\mathcal H_r^\psi := \big\{x\mapsto\psi\big((W_0+\Delta W)x\big)\ :\ \Delta W\in\mathcal D_r\big\},
\end{equation}
with $\mathcal D_r=\mathcal B_{r,\rho}$ (the rank-sensitive relaxation) as the default when $\mathcal D_r$ is not specified.
\end{definition}

The transfer of Lemma~\ref{lem:matrix-rad} to $\mathcal H_r^\psi$, the resulting complexity bounds for both constraint sets, and their scope are carried out in Appendix~\ref{app:remarks} (Corollary~\ref{cor:rank-dep}, Remark~\ref{rem:linear-readout}), using the auxiliary translation and contraction lemmas of Appendix~\ref{app:aux}.

\section{Q2: Spectral Structure and the Threshold Rank of Cancellation}
\label{sec:spectral}

Throughout this section, and wherever the spectral quantities $\Sigma_S$, $\Sigma_T$, $\Delta\Sigma$, $\tau_r$ appear (including Theorem~\ref{thm:cone}(ii) and Corollary~\ref{cor:spectral-route}), $D_S$ and $D_T$ are assumed to have finite second moments, so that these matrices are well defined. Let $W:=W_0+\Delta W$ be the adapted weight. For a target with second-moment matrix $\Sigma_T$, the \emph{alignment error} of $W$ is $\|W\Delta\Sigma W^T\|_2$; this orientation is the dimensionally consistent one for $m\times d$ weights, since $W\Delta\Sigma W^T$ is the second-moment discrepancy of the layer outputs $Wx$ across domains.

\subsection{Cancellation and the Tail Profile}

The base construction is exact cancellation of the top spectral block of $W_0$. Its realizability under per-factor Frobenius budgets is governed by a nuclear-norm condition---which, by Lemma~\ref{lem:factorization}, is the only condition there is. We record the construction as Lemma~\ref{lem:spectral} in Appendix~\ref{app:aux}, since it is verification rather than substance, and use here only its conclusion: if
\begin{equation}
\label{eq:nuclear-budget}
\rho \ \ge\ \|(W_0)_{\le r}\|_* = \sum_{i=1}^r s_i(W_0),
\end{equation}
then $\Delta W:=-(W_0)_{\le r}$ lies in $\mathcal A_r$ and yields $W=(W_0)_{>r}$, hence
\begin{equation}
\|W\Delta\Sigma W^T\|_2 = \big\|(W_0)_{>r}\,\Delta\Sigma\,(W_0)_{>r}^T\big\|_2 \ =:\ \tau_r.
\end{equation}
Here $(W_0)_{\le r}$ and $(W_0)_{>r}$ are the top-$r$ block and residual tail of the SVD of $W_0$. Throughout, we fix one SVD $W_0=\sum_is_i(W_0)u_iv_i^T$ once and for all: when singular values repeat, the blocks and all derived quantities ($\tau_r$, the cancellation updates, the vectors $u_i,v_i$) are defined relative to this fixed choice.

Condition~\eqref{eq:nuclear-budget} is where the rank finally does something: the budget needed grows with $r$, so a larger $r$ makes larger cancellations \emph{affordable}. This is the affordability role announced in Section~\ref{sec:intro}, and it is the exact complement of the capacity role ruled out in Section~\ref{sec:complexity}.

The tail profile is nonincreasing, with $\tau_r=0$ for $r\ge\mathrm{rank}(W_0)$ (Lemma~\ref{lem:monotone}, Appendix~\ref{app:aux}); this monotonicity is what makes the spectral threshold below a genuine threshold rather than merely the smallest element of an unstructured set; without it, ``$\min\{r:\tau_r\le\epsilon/\kappa\}$'' would not be a rank one could search for by increasing $r$.

\subsection{From Alignment Error to Discrepancy}

The link between alignment error and distributional discrepancy is an \emph{assumption}, not a theorem, and we flag it as such. Because the assumption is invoked in our proofs only at the cancellation weights $W=(W_0)_{>r'}$, $r'\in\mathbb Z_{\ge0}$, of Lemma~\ref{lem:spectral}, we state it for exactly that family; Remark~\ref{rem:kappa-scope} (Appendix~\ref{app:remarks}) explains why the restriction matters. Example~\ref{ex:gaussian} in Appendix~\ref{app:aux} exhibits a concrete Gaussian model, with bounded linear readouts supported on the range (output space) of the adapted weight, in which the assumption provably holds uniformly along the entire cancellation family with the explicit constant $\kappa=\tfrac2c\sqrt{2/\pi}$, $c=\lambda_{\min}(\Sigma_T)^{1/2}\,s_{\mathrm{rank}(W_0)}(W_0)$, for $W_0\ne0$ (the degenerate case $W_0=0$ is handled trivially in the example). That instantiation uses an explicitly weight-dependent readout family $\{\Psi_W\}$, so it should be read as a restricted model in which the assumption is verifiable, not as a validation of the assumption for an arbitrary fixed readout class.

\begin{assumption}[Alignment Control Along the Cancellation Family]
\label{ass:kappa}
There exists $\kappa>0$ such that for every integer $r'\in\mathbb Z_{\ge0}$ and the corresponding cancellation weight $W=(W_0)_{>r'}$,
\begin{equation}
\tilde d_{\mathcal H_W}(D_S,D_T) \ \le\ \kappa\,\|W\Delta\Sigma W^T\|_2,
\end{equation}
with $\mathcal H_W$ the induced readout class of Definition~\ref{def:induced}.
\end{assumption}


\begin{theorem}[Spectral Threshold of the Cancellation Certificate: Upper Bounds]
\label{thm:minrank}
Assume $W_0\ne0$ and that $D_S,D_T$ have finite second moments, write $K:=\mathrm{rank}(W_0)$, fix a target $D_T$ (hence $\Delta\Sigma$) and a tolerance $\epsilon>0$, and suppose Assumption~\ref{ass:kappa} holds. Define the \emph{spectral threshold rank}
\begin{equation}
r_{\mathrm{spec}}(\epsilon) := \min\big\{1\le r\le K : \tau_r \le \epsilon/\kappa\big\},
\end{equation}
which is well defined (since $\tau_K=0$) and is a threshold by Lemma~\ref{lem:monotone}. If moreover the budget covers the corresponding top block,
\begin{equation}
\label{eq:budget-at-rspec}
\rho\ \ge\ \sum_{i=1}^{r_{\mathrm{spec}}(\epsilon)}s_i(W_0),
\end{equation}
then the top-block cancellation at rank $r_{\mathrm{spec}}(\epsilon)$ is admissible and certifies $\tilde d_{\mathcal H_W}(D_S,D_T)\le\epsilon$ for $W=(W_0)_{>r_{\mathrm{spec}}(\epsilon)}$. Two cases. If $\|\Delta\Sigma\|_2=0$, then $\tau_r=0$ for every $r$ and $r_{\mathrm{spec}}(\epsilon)=1$, the cancellation certificate remaining conditional on the budget condition \eqref{eq:budget-at-rspec}. If instead $\|\Delta\Sigma\|_2>0$, then $\tau_r\le s_{r+1}(W_0)^2\|\Delta\Sigma\|_2$, and consequently:
\begin{itemize}
\item if $s_i(W_0)\le\bar C\,i^{-\alpha}$ for some $\alpha>0$ (polynomial decay),
\begin{equation}
r_{\mathrm{spec}}(\epsilon) \ \le\ \min\left\{K,\ \max\left\{1,\ \Big\lceil\big(\kappa\,\bar C^2\,\|\Delta\Sigma\|_2/\epsilon\big)^{1/(2\alpha)}\Big\rceil\right\}\right\};
\end{equation}
\item if $s_i(W_0)\le\bar C\beta^{\,i}$ for some $\beta\in(0,1)$ (geometric decay),
\begin{equation}
r_{\mathrm{spec}}(\epsilon) \ \le\ \min\left\{K,\ \max\left\{1,\ \Big\lceil\frac{\log\!\big(\kappa\,\bar C^2\,\|\Delta\Sigma\|_2/\epsilon\big)}{2\log(1/\beta)}\Big\rceil\right\}\right\}.
\end{equation}
\end{itemize}
\end{theorem}

The relation between the spectral threshold and the budget-feasible certified rank---in particular, that the budget decides \emph{whether} any rank certifies, never \emph{which} rank is minimally certifying---is developed in Remark~\ref{rem:rcert} (Appendix~\ref{app:remarks}).

Theorem~\ref{thm:minrank} is a \emph{sufficiency} statement about one specific certificate. Necessity---within that certificate family---requires lower bounds on $\tau_r$, which cannot follow from one-sided decay assumptions; we supply them under two-sided conditions.

\begin{proposition}[Lower Bounds for the Cancellation Certificate]
\label{prop:lower}
Assume $D_S,D_T$ have finite second moments. Let $u_i,v_i$ be the left and right singular vectors of $W_0$ and $K=\mathrm{rank}(W_0)$. Suppose there is $c_1>0$ with the \emph{uniform directional alignment condition}
\begin{equation}
|v_i^T\Delta\Sigma\,v_i|\ \ge\ c_1\qquad\text{for every }i\in\{2,\dots,K\}
\end{equation}
(necessarily $c_1\le\|\Delta\Sigma\|_2$, and the condition presupposes $K\ge2$). Then for every $r$ with $1\le r\le K-1$,
\begin{equation}
\tau_r \ \ge\ c_1\,s_{r+1}(W_0)^2.
\end{equation}
Consequently, uniformly throughout the \emph{non-saturated regime} $r_{\mathrm{spec}}(\epsilon)<K$: if additionally $s_i(W_0)\ge c_0\,i^{-\alpha}$ for the indices under consideration (two-sided polynomial decay),
\begin{equation}
r_{\mathrm{spec}}(\epsilon) \ \ge\ \big(\kappa\,c_1c_0^2/\epsilon\big)^{1/(2\alpha)} - 1,
\end{equation}
and if $s_i(W_0)\ge c_0\beta^{\,i}$, then $r_{\mathrm{spec}}(\epsilon)\ge\tfrac1{2\log(1/\beta)}\log(\kappa c_1c_0^2/\epsilon)-1$. 
\end{proposition}

These lower bounds have the same dependence on $\epsilon$ and on the decay as the upper bounds of Theorem~\ref{thm:minrank}, but with constants governed by the directional quantity $c_1$ rather than by $\|\Delta\Sigma\|_2$; within the top-singular-block cancellation family the two match in order precisely when $c_1\asymp\|\Delta\Sigma\|_2$, i.e.\ when the discrepancy $\Delta\Sigma$ is not nearly orthogonal, in the quadratic-form sense, to the tail singular directions. We emphasize the scope: these are lower bounds on the threshold of \emph{this certificate family only}, valid away from rank saturation; they do not preclude other rank-$r$ updates $W_0+\Delta W$ from certifying the tolerance by different means at smaller rank. If the directional condition is available only at the single index $r_{\mathrm{spec}}(\epsilon)+1$, the same conclusions hold as an a posteriori statement at that index rather than as a uniform rate.
\subsection{Generalization Bounds: Transport and Spectral Routes}

\begin{theorem}[LoRA Generalization Bound via Wasserstein Transport]
\label{thm:gen}
Let $\mathcal H=\mathcal H_r^\psi:=\mathcal H_{r,\mathcal D_r}^\psi$ be the fixed-readout class of Definition~\ref{def:fixed-readout}, built on either constraint set $\mathcal D_r\in\{\mathcal A_r,\ \mathcal B_{r,\rho}\}$ (all statements below refer to the selected $\mathcal D_r$), let Assumptions~\ref{ass:bddfeat} and~\ref{ass:bounded} hold, let $D_T$ have a finite first moment, let the loss be the absolute loss, set $M:=\|W_0\|_2+\rho$, and suppose $n_S\ge\log(m+d)$. Then with probability at least $1-\delta$ over $n_S$ i.i.d.\ source samples, every $h\in\mathcal H_r^\psi$ satisfies
\begin{equation}
\begin{split}
\epsilon_T(h) \le \ & \hat\epsilon_S(h) \\
&+ 2\sqrt{2}\,L_\psi\,\rho\,X_{\max}\sqrt{\frac{m}{n_S}} \\
&+ b\sqrt{\frac{\log(1/\delta)}{2n_S}} \\
&+ 2L_\psi M\, W_1(D_S,D_T) + \lambda^*,
\end{split}
\end{equation}
where $\lambda^*=\inf_{h'\in\mathcal H_r^\psi}[\epsilon_S(h')+\epsilon_T(h')]$ and $\hat\epsilon_S$ is the empirical source risk defined in Section~\ref{sec:framework}. The complexity term is the rank-free worst case of Corollary~\ref{cor:rank-dep}; its data-dependent, rank-sensitive refinement is discussed in Remark~\ref{rem:gen-Ar}.
\end{theorem}

\begin{corollary}[Spectral Route: Population and Empirical Bounds]
\label{cor:spectral-route}
Suppose $W_0\ne0$, $D_S$ and $D_T$ have finite second moments, Assumption~\ref{ass:kappa} holds, and the budget condition \eqref{eq:nuclear-budget} holds at rank $r$. Let $W=(W_0)_{>r}$ be the cancellation weight, $\mathcal H_W$ its induced readout class, and $\lambda^*_W:=\inf_{h'\in\mathcal H_W}[\epsilon_S(h')+\epsilon_T(h')]$. Then
\begin{equation}
\tilde d_{\mathcal H_W}(D_S,D_T)\ \le\ \kappa\,\tau_r,
\end{equation}
and:
\begin{enumerate}
\item[(i)] \emph{(Population.)} For every $h\in\mathcal H_W$,
\begin{equation}
\epsilon_T(h)\ \le\ \epsilon_S(h) + \kappa\,\tau_r + \lambda^*_W.
\end{equation}
\item[(ii)] \emph{(Empirical.)} If in addition Assumption~\ref{ass:bounded} holds for $\mathcal H_W$, then with probability at least $1-\delta$ over $n_S$ i.i.d.\ source samples, every $h\in\mathcal H_W$ satisfies
\begin{equation}
\epsilon_T(h)\ \le\ \hat\epsilon_S(h) + 2\,\bar{\mathfrak R}_{n_S}(\mathcal H_W) + b\sqrt{\frac{\log(1/\delta)}{2n_S}} + \kappa\,\tau_r + \lambda^*_W,
\end{equation}
with $\bar{\mathfrak R}_{n_S}$ the expected Rademacher complexity of Section~\ref{sec:framework}. If moreover Assumption~\ref{ass:bddfeat} holds and $\Psi_W$ consists of linear readouts of norm at most one (e.g.\ the range-adapted family of Example~\ref{ex:gaussian}), then
\begin{equation}
\bar{\mathfrak R}_{n_S}(\mathcal H_W)\ \le\ \frac{\|W\|_2\,X_{\max}}{\sqrt{n_S}}\ =\ \frac{s_{r+1}(W_0)\,X_{\max}}{\sqrt{n_S}}:
\end{equation}
cancelling a larger spectral block also shrinks the certificate class's own complexity.
\end{enumerate}
In particular, at $r=r_{\mathrm{spec}}(\epsilon)$ (feasible under \eqref{eq:budget-at-rspec}) the discrepancy term is at most $\epsilon$. This is the route by which the spectral machinery of Theorem~\ref{thm:minrank} enters a genuine sample-based generalization statement; it is complementary to, not a consequence of, the transport route of Theorem~\ref{thm:gen}.
\end{corollary}

The proof, which combines Lemma~\ref{lem:spectral}, Assumption~\ref{ass:kappa}, and the concentration step of Theorem~\ref{thm:gen}, is given in Appendix~\ref{app:proofs}.




Variants of the complexity term and the two distinct roles of the frozen weight are discussed in Remarks~\ref{rem:gen-Ar} and~\ref{rem:w0-roles} (Appendix~\ref{app:remarks}).

\section{Q3: Geometry of the Transferability Cone}
\label{sec:cone}

The divergence-defined cone of Definition~\ref{def:cone} turns out to carry no geometric information whenever its kernel premise below holds.

\begin{proposition}[The Divergence-Defined Cone Has Ambient-Scale Radius]
\label{prop:trivial-radius}
Suppose some admissible $W=W_0+\Delta W$, $\Delta W\in\mathcal A_r$, has a nontrivial kernel---automatic when $m<d$, and in the square case satisfiable whenever the budget permits a singular reachable $W$, e.g.\ via the cancellation of Lemma~\ref{lem:spectral} or the rank-one annihilation of Proposition~\ref{prop:tight}(i). Then, for every $(r,\text{budget})$ pair satisfying this kernel premise:
\begin{enumerate}
\item[(i)] If $\mathcal X=\mathbb R^d$ and $D_S$ has a finite second moment, $\sup_{D_T\in\mathcal C_r(W_0,D_S;\epsilon)}W_2(D_S,D_T)=\infty$ for every $\epsilon\ge0$.
\item[(ii)] If $\mathcal X=\{x:\|x\|_2\le X_{\max}\}$, then for suitable $D_S$ supported in $\mathcal X$ the same supremum is at least $X_{\max}/2$, i.e.\ of the order of the diameter of $\mathcal X$, again for every $\epsilon\ge0$.
\end{enumerate}
\end{proposition}

In both cases the radius is set by the ambient support---infinite in (i), a support-scale lower bound $X_{\max}/2$ in (ii)---and is insensitive to $\epsilon$, $r$, and $\rho$ (within the premise); it therefore says nothing about the adaptation. A meaningful radius requires restricting to targets \emph{realizable by the adaptation itself}.

\begin{definition}[Two-Witness Pushforward Cone]
\label{def:pushcone}
Assume $m=d$, so that $I+\Delta W$ acts on $\mathcal X$, take $\mathcal X$ closed under the maps $x\mapsto(I+\Delta W)x$, $\Delta W\in\mathcal A_r$ (e.g.\ $\mathcal X=\mathbb R^d$), and assume $D_S$ has a finite second moment (so that all pushforwards below do too). Set
\begin{equation}
\mathcal C_{r,\mathrm{two}}^{\mathrm{push}}(\epsilon) := \big\{(I+\Delta W)_\#D_S : \Delta W\in\mathcal A_r\big\}\ \cap\ \mathcal C_r(W_0,D_S;\epsilon),
\end{equation}
with radius $\mathrm{rad}\big(\mathcal C_{r,\mathrm{two}}^{\mathrm{push}}(\epsilon)\big) := \sup_{D_T\in\mathcal C_{r,\mathrm{two}}^{\mathrm{push}}(\epsilon)}W_2(D_S,D_T)$. The subscript records that this is the \emph{two-witness} pushforward cone: the update realizing the pushforward and the update certifying transferability need not coincide (Remark~\ref{rem:two-witnesses}). The $\mathcal B$-based variant is defined identically with $\mathcal A_r$ replaced by $\mathcal B_{r,\rho}$ in \emph{both} occurrences of the update set, requiring additionally that $\mathcal X$ be closed under $x\mapsto(I+\Delta W)x$ for $\Delta W\in\mathcal B_{r,\rho}\supseteq\mathcal A_r$ (automatic for $\mathcal X=\mathbb R^d$):
\begin{equation}
\mathcal C_{r,\mathrm{two}}^{\mathrm{push},\mathcal B}(\epsilon):=\big\{(I+\Delta W)_\#D_S:\Delta W\in\mathcal B_{r,\rho}\big\}\cap\mathcal C_r^{\mathcal B}(W_0,D_S;\epsilon),
\end{equation}
where $\mathcal C_r^{\mathcal B}$ denotes Definition~\ref{def:cone} with $\mathcal A_r$ replaced by $\mathcal B_{r,\rho}$.
\end{definition}

Membership in $\mathcal C_{r,\mathrm{two}}^{\mathrm{push}}(\epsilon)$ involves two existential witnesses---one update realizes the target as a pushforward, another certifies the discrepancy condition---and they need not coincide; indeed, our tightness construction (Proposition~\ref{prop:tight}) uses two different rank-one updates. The single-update variant of the cone, to which the upper bound of Theorem~\ref{thm:cone}(i) applies verbatim, is discussed in Remark~\ref{rem:two-witnesses} (Appendix~\ref{app:remarks}).

\begin{theorem}[Geometry of the Two-Witness Pushforward Cone]
\label{thm:cone}
Let $\tau_{\mathrm{tail}} := \|(W_0)_{>r}\Delta\Sigma(W_0)_{>r}^T\|_2$ for a given target.
\begin{itemize}
\item[(i)] \emph{(Reachable radius: rank-independent, and sharp; $m=d$ and $D_S$ with finite second moment, as in Definition~\ref{def:pushcone}.)}
\begin{equation}
\mathrm{rad}\big(\mathcal{C}_{r,\mathrm{two}}^{\mathrm{push}}(\epsilon)\big) \le \sqrt{\lambda_{\max}(\Sigma_S)}\cdot \rho,
\end{equation}
with no dependence on $r$. The bound is sharp: if the budget covers the weakest input direction of the frozen weight, $s_d(W_0)\le\rho$---automatic whenever $W_0$ is singular---it is attained with equality for an isotropic Gaussian source by a rank-one update (Proposition~\ref{prop:tight}), so under that condition the constant cannot be improved, whether by restricting $r$ or otherwise. (This is a sharpness statement over source distributions, not an equality for every fixed source.) The same bound and witness hold verbatim for the $\mathcal B$-based two-witness cone $\mathcal C_{r,\mathrm{two}}^{\mathrm{push},\mathcal B}(\epsilon)$ of Definition~\ref{def:pushcone}.
\item[(ii)] \emph{(Membership via cancellation; general $m,d$, with $D_S,D_T$ of finite second moments.)} Under Assumption~\ref{ass:kappa}, any $D_T$ with $\tau_{\mathrm{tail}}\le\epsilon/\kappa$ lies in $\mathcal C_r(W_0,D_S;\epsilon)$, provided the nuclear budget \eqref{eq:nuclear-budget} holds. Rank enters genuinely here: the budget $\sum_{i\le r}s_i(W_0)$ needed to satisfy the premise grows with $r$, so a larger $r$ is what makes larger cancellations affordable---not what enlarges the radius a fixed budget attains.
\end{itemize}
\end{theorem}

The Wasserstein displacement produced by the cancellation update itself is recorded as Proposition~\ref{prop:cancel-dist} in Appendix~\ref{app:aux}.

Remark~\ref{rem:same-phenomenon} (Appendix~\ref{app:remarks}) explains how the statistical and geometric collapses are two readings of Lemma~\ref{lem:extremal}, and how rank-dependence is recovered.

\section{Numerical Illustrations}
\label{sec:experiments}
 
The results above are theorems; the purpose of this section is not to
add evidence for them but to exhibit the predicted phenomena at
observable scale, to check that the regimes and constants the theory
identifies are the ones that actually appear, and to probe empirically
the one question the theory leaves open (the mechanism selecting
low-nuclear-norm solutions). All experiments are deliberately
synthetic, linear, and single-layer---the exact setting of the
theory---and make no claim about LoRA fine-tuning of deep networks.
Unless stated otherwise, results are averaged over five seeds and
error bars denote one standard deviation. %
 
\paragraph*{Common setup}
We use a teacher--student regression $y=(W_0+\Delta^\ast)x+\xi$ with
$x\sim\mathcal N(0,I_d)$, $\xi\sim\mathcal N(0,\sigma^2I_m)$, a frozen
Gaussian $W_0$ ($W_{0,ij}\sim\mathcal N(0,1/d)$), and a rank-$4$
teacher update $\Delta^\ast$ with singular values
$(0.8,0.6,0.4,0.2)$. The student trains $\Delta W=BA$ with the
\emph{hard per-factor projection}: after every optimizer step
(AdamW, learning rate $10^{-2}$, no weight decay), $B$ and $A$ are
independently rescaled onto their Frobenius balls with
$B_B=B_A=\sqrt\rho$, so that the iterates remain in $\mathcal A_r$
exactly (Lemma~\ref{lem:factorization}). Losses are summed over output
coordinates and averaged over samples.
 
\subsection{Rank-One Extremality of the Worst Case}
\label{sec:exp-extremal}
 
We first test Lemma~\ref{lem:extremal} directly. On a fixed instance
($m=d=32$, $\sigma=0.5$, $n_{\mathrm{train}}=64$,
$n_{\mathrm{test}}=5000$) we \emph{maximize} the empirical
generalization gap
$\mathrm{gap}(\Delta W)=\hat\epsilon_{\mathrm{test}}-\hat\epsilon_{\mathrm{train}}$
over $\Delta W\in\mathcal A_r$ by projected gradient ascent ($5000$
steps, five random restarts, best kept), for
$r\in\{1,2,4,8,16,32\}$ and $\rho\in\{1,4\}$.
Figure~\ref{fig:bprime} shows the outcome. The maximized gap is flat
in $r$ (range/mean $9.7\%$ at $\rho=1$, $4.1\%$ at $\rho=4$; $r=1$
attains $91\%$--$96\%$ of the best value found at any rank, the
residual spread being ascent slack), the maximizer saturates the
nuclear budget exactly ($\|BA\|_*=\rho$ to four decimals), and its
spectral-entropy effective rank is $1.0000$ at \emph{every} nominal
rank. To rule out the implicit low-rank bias of small
initialization as a confound, we repeated the ascent with a
full-rank Kaiming-normal initialization of $B$: the maximized gaps
change by at most $2.2\%$ and the maximizer still collapses to
effective rank $1.00$ throughout. The worst case over $\mathcal A_r$
is rank-one, independently of $r$ and of initialization---exactly
Lemma~\ref{lem:extremal}.
 
\begin{figure*}[t]
\centering
\includegraphics[width=\textwidth]{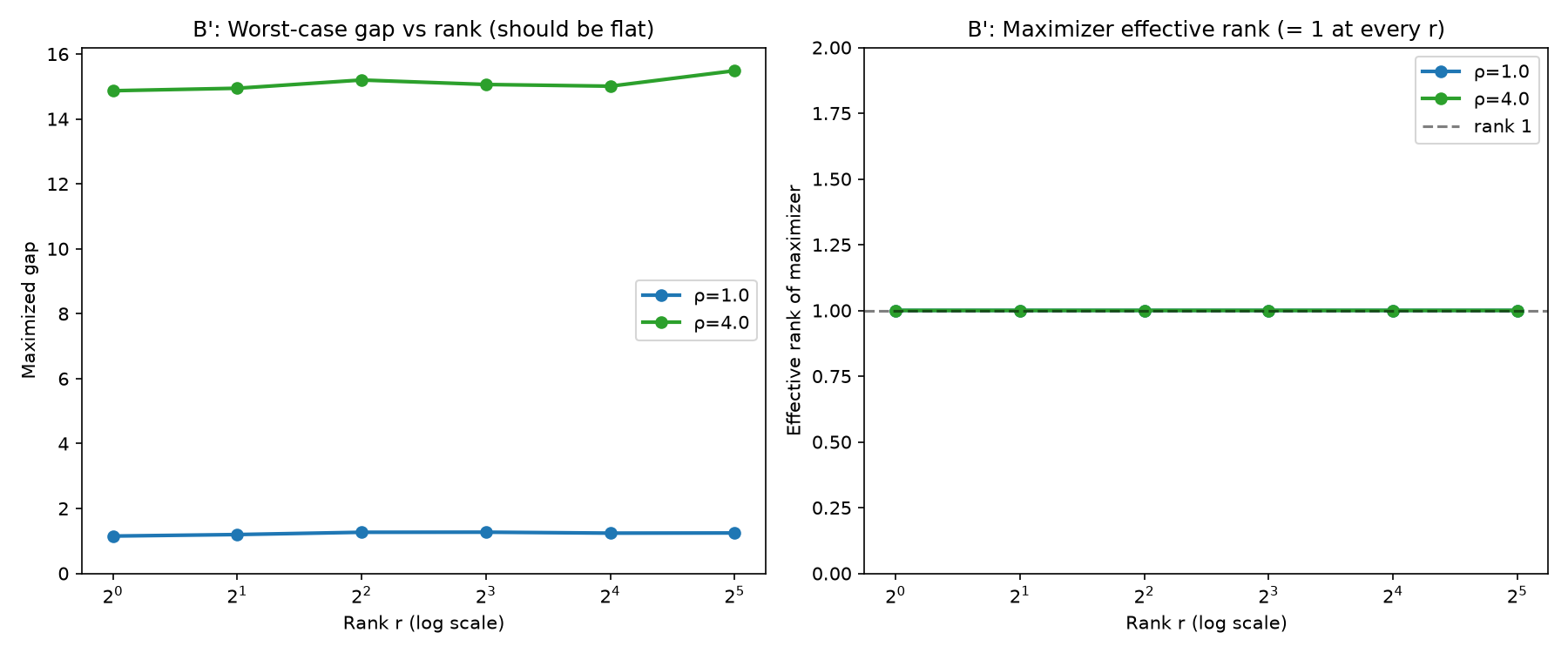}
\caption{Worst-case gap maximization over $\mathcal A_r$. Left: the
maximized gap is flat in the nominal rank $r$ for both budgets.
Right: the effective rank of the maximizing update is $1.00$ at every
$r$. The supremum is attained in $\mathcal A_1$, as predicted by
Lemma~\ref{lem:extremal}.}
\label{fig:bprime}
\end{figure*}
 
\subsection{The ERM Gap: Affordability, then Rank-Freeness}
\label{sec:exp-gap}
 
We next train by empirical risk minimization on the same family of
instances and record the realized gap as a function of $(r,\rho)$.
Figure~\ref{fig:two-regime} shows two regimes, cleanly separated by
whether the learned update saturates its budget. While
$\|BA\|_*<\rho$ (small $r$: the rank cap prevents the model from
spending the budget), the gap grows with $r$---this is the
\emph{affordability} regime, where the strict inclusions
$\mathcal A_1\subsetneq\mathcal A_2\subsetneq\cdots$ are the operative
constraint. Once $\|BA\|_*\approx\rho$, the gap plateaus: additional
rank buys no additional overfitting. The plateau level is linear in
the budget, $\mathrm{gap}\approx0.83\,\rho$ with $R^2=0.9998$
(Figure~\ref{fig:plateau}), matching the linear $\rho$-dependence of
the complexity in Proposition~\ref{prop:matrix-collapse}.
 
Because this two-regime reading was formulated after inspecting the
first run, we re-ran the experiment as a pre-registered confirmation
on fresh seeds and different dimensions ($m=d=64$,
$n_{\mathrm{train}}=128$, $r$ up to $64$): the plateau-onset and
linearity predictions hold out of sample
($\mathrm{gap}\approx0.93\,\rho$, $R^2=0.9986$), with the deep-plateau
values flat to under $1\%$ at the smaller budgets. Two reproducible
second-order effects deserve note, both benign for the theory: at the
largest budget the gap \emph{overshoots} slightly at the first
budget-saturating rank before settling onto the plateau, and it then
\emph{declines} mildly (by $\sim$10\%) as $r$ grows further---extra
rank spreads a fixed nuclear budget over more directions and mildly
regularizes the ERM solution. Neither effect is in tension with the
results above: the worst case over the class is exactly flat
(Section~\ref{sec:exp-extremal}); these are properties of the
particular solution ERM selects inside it.
 
\begin{figure}[t]
\centering
\includegraphics[width=\columnwidth]{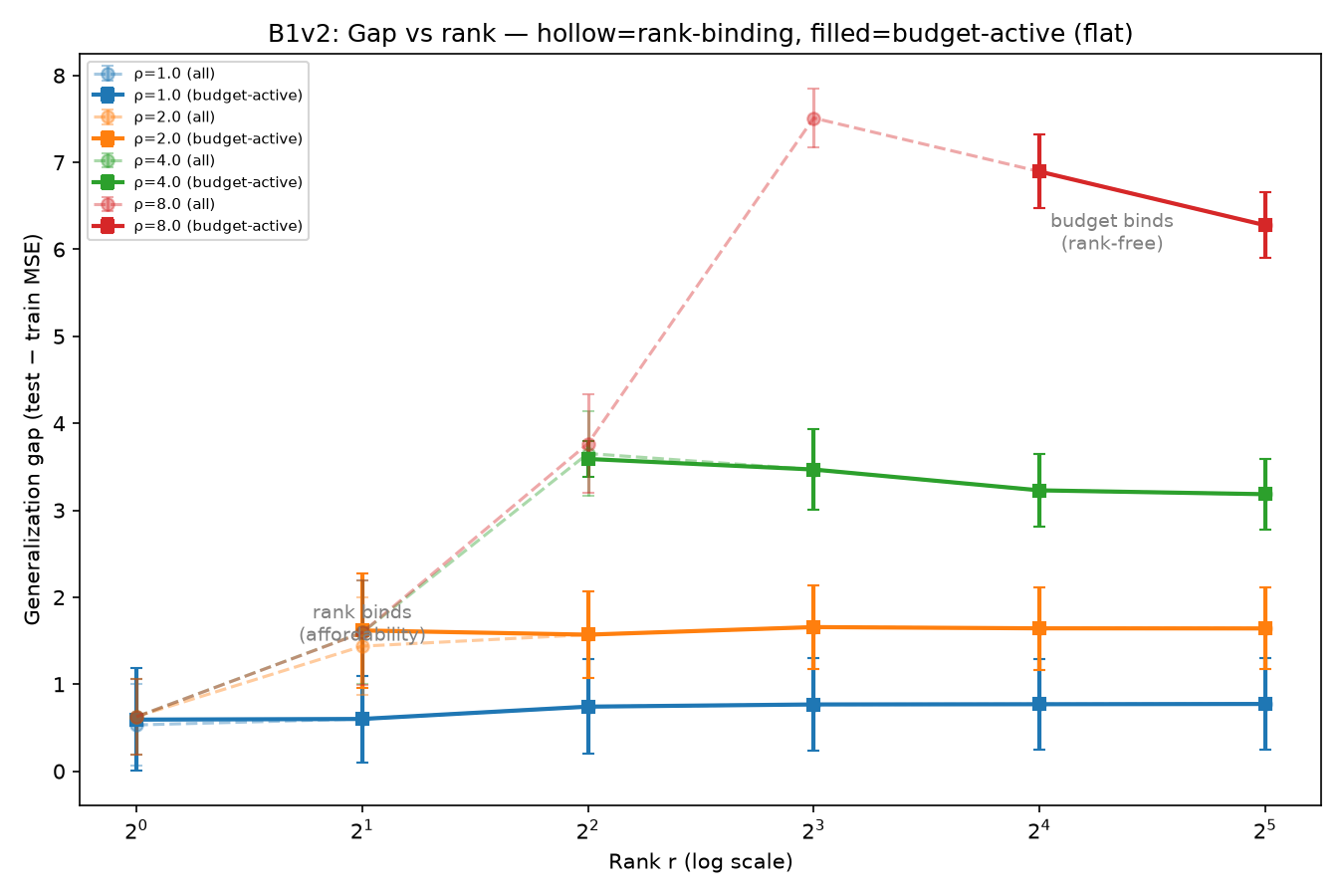}
\caption{Realized ERM generalization gap versus nominal rank, per
budget $\rho$ ($m=d=32$, $n_{\mathrm{train}}=64$, five seeds). Solid
markers: budget-active configurations ($\|BA\|_*>0.95\rho$). The gap
grows only while the rank cap prevents the budget from being spent
(affordability regime) and is flat thereafter (rank-free regime).}
\label{fig:two-regime}
\end{figure}
 
\begin{figure}[t]
\centering
\includegraphics[width=\columnwidth]{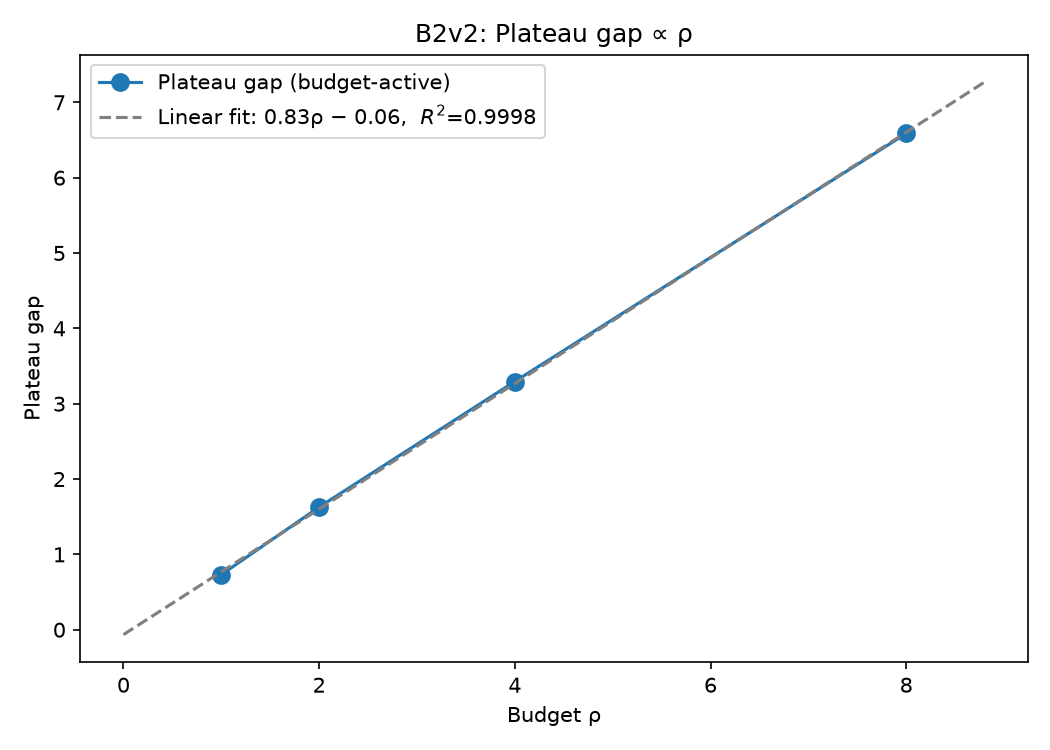}
\caption{Plateau gap (averaged over budget-active ranks) versus
budget $\rho$: linear with $R^2=0.9998$, as the $\rho$-linear
complexity of Proposition~\ref{prop:matrix-collapse} predicts.}
\label{fig:plateau}
\end{figure}
 
\subsection{Spectral Threshold Scaling and the Role of Alignment}
\label{sec:exp-threshold}
 
We verify the scaling laws of Theorem~\ref{thm:minrank} and
Proposition~\ref{prop:lower} without any training: for
$W_0=U\,\mathrm{diag}(s)\,V^T\in\mathbb R^{512\times512}$ with random
orthogonal factors and prescribed spectra, we compute the tail
profile $\tau_r$ exactly (in double precision) and read off
$r_{\mathrm{spec}}(\epsilon)=\min\{r:\tau_r\le\epsilon\}$ (taking
$\kappa=1$, which only rescales $\epsilon$), over $40$ tolerances per
decade range. We use two constructions of
$\Delta\Sigma$: an \emph{aligned} one,
$\Delta\Sigma=V\,\mathrm{diag}(\delta)\,V^T$ with
$|\delta_i|\in[0.5,1.5]$ and random signs, which enforces the
directional alignment condition of Proposition~\ref{prop:lower} with
$c_1\ge0.5$; and a \emph{random} one, the difference of two random
covariances, for which $c_1=O(1/d)\approx2\times10^{-4}$ at
$d=512$---i.e., the misaligned regime in which the lower bound of
Proposition~\ref{prop:lower} degenerates and only the one-sided upper
bound of Theorem~\ref{thm:minrank} applies.
 
Table~\ref{tab:cfits} and Figure~\ref{fig:cthreshold} report the
fitted scaling exponents (fit window $4\le r_{\mathrm{spec}}\le256$,
excluding integer-rounding and saturation effects). Under alignment,
the polynomial exponents match $1/(2\alpha)$ to within $7$--$9\%$ and
the geometric slopes match $1/(2\log(1/\beta))$ to within $1.3\%$.
Without alignment the polynomial exponents inflate by
$18$--$54\%$---more rank is needed than the matched rate
predicts---while the geometric case is insensitive to alignment,
because under exponential decay the top tail term dominates
$\tau_r$ regardless of the direction of $\Delta\Sigma$. The
systematic slight upward bias of the aligned fits is the integer
rounding of $r_{\mathrm{spec}}$ together with the $\max_i$ over the
random weights $|\delta_i|$ in the tail. This is a direct empirical
demonstration that the alignment condition in
Proposition~\ref{prop:lower} is load-bearing rather than technical:
it is precisely what pins the threshold to the matched rate.
 
\begin{table}[t]
\caption{Fitted versus predicted scaling of
$r_{\mathrm{spec}}(\epsilon)$ ($d=512$, five seeds, mean$\pm$std).
Polynomial decay: slope of $\log r_{\mathrm{spec}}$ on
$\log(1/\epsilon)$, predicted $1/(2\alpha)$. Geometric decay: slope
of $r_{\mathrm{spec}}$ on $\log(1/\epsilon)$, predicted
$1/(2\log(1/\beta))$.}
\label{tab:cfits}
\centering
\begin{tabular}{llccc}
\toprule
Decay & $\Delta\Sigma$ & Predicted & Fitted & Error \\
\midrule
$i^{-0.5}$ & aligned & $1.000$ & $1.091\pm0.06$ & $9.1\%$ \\
$i^{-0.5}$ & random  & $1.000$ & $1.538\pm0.05$ & $53.8\%$ \\
$i^{-1}$   & aligned & $0.500$ & $0.537\pm0.01$ & $7.4\%$ \\
$i^{-1}$   & random  & $0.500$ & $0.653\pm0.02$ & $30.6\%$ \\
$i^{-2}$   & aligned & $0.250$ & $0.270\pm0.01$ & $8.0\%$ \\
$i^{-2}$   & random  & $0.250$ & $0.294\pm0.01$ & $17.7\%$ \\
\midrule
$0.95^i$ & random & $9.748$ & $9.737\pm0.05$ & $0.1\%$ \\
$0.9^i$  & random & $4.746$ & $4.749\pm0.03$ & $0.1\%$ \\
$0.8^i$  & random & $2.241$ & $2.212\pm0.03$ & $1.3\%$ \\
\bottomrule
\end{tabular}
\end{table}
 
\begin{figure}[t]
\centering
\includegraphics[width=\columnwidth]{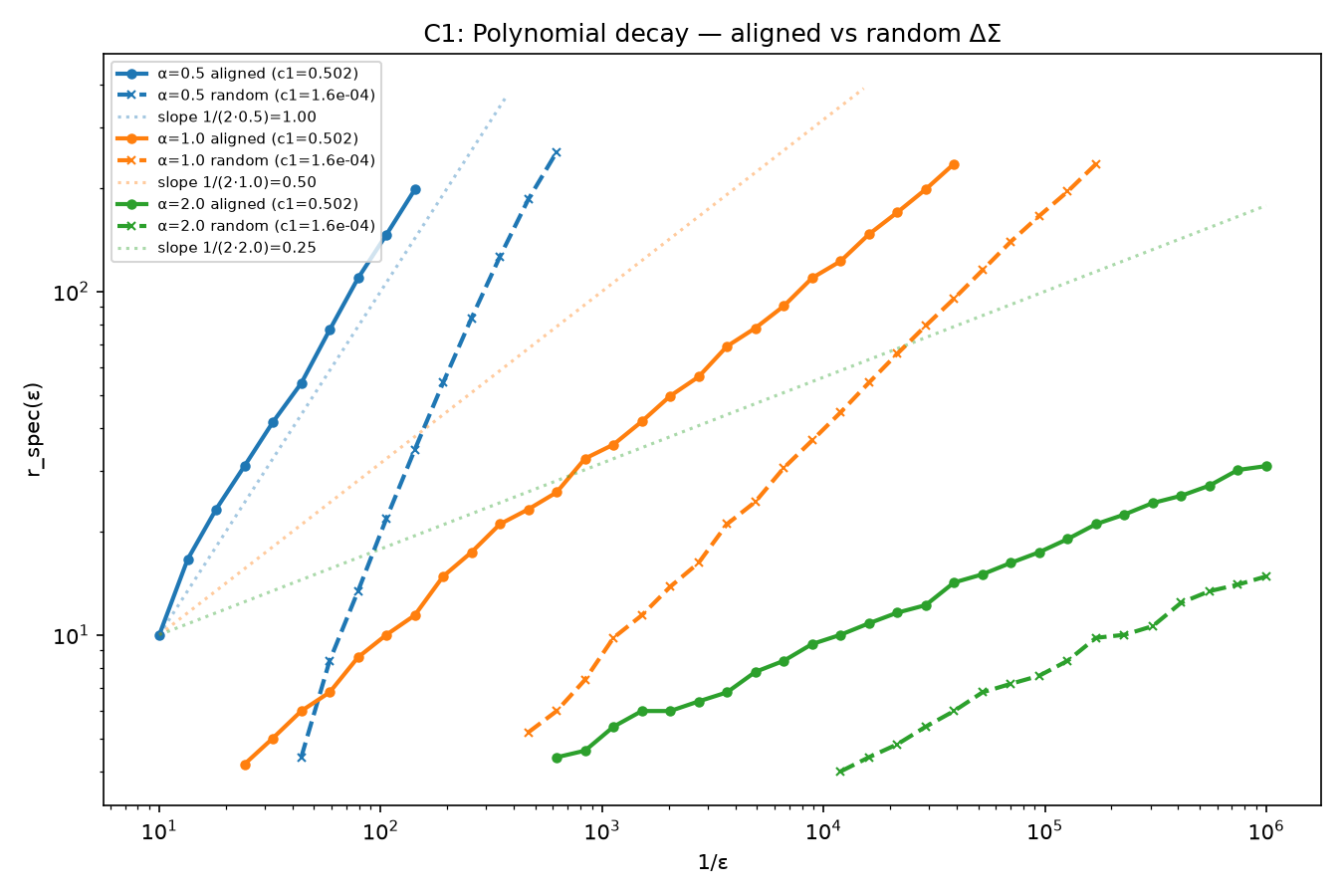}
\caption{Threshold rank $r_{\mathrm{spec}}(\epsilon)$ versus
$1/\epsilon$ under polynomial spectral decay, aligned versus random
$\Delta\Sigma$, with the predicted slopes $1/(2\alpha)$ (dashed).
Alignment on: the matched rate of
Theorem~\ref{thm:minrank}/Proposition~\ref{prop:lower}. Alignment
off: only the upper bound applies and the realized exponent inflates.}
\label{fig:cthreshold}
\end{figure}
 
\subsection{The Complexity Branches Across Feature Spectra}
\label{sec:exp-branches}
 
Both complexities in Section~\ref{sec:complexity} are exact values of
computable suprema, so we evaluate them directly by Monte Carlo
($200$ sign draws; $n=512$, $d=256$, $m=64$, $\rho=1$): the per-factor
value $\tfrac\rho n\mathbb E\|S\|_2$
(Proposition~\ref{prop:matrix-collapse}), the joint-budget value
$\tfrac\rho n\mathbb E\|S_r\|_F$ (Lemma~\ref{lem:matrix-rad}), and the
rank-free ceiling $\tfrac\rho n\mathbb E\|S\|_F$, on three feature
distributions: well-spread ($r_{\mathrm{eff}}(G_n)\approx89$),
spectrally decaying ($r_{\mathrm{eff}}\approx6$), and near-collinear
($r_{\mathrm{eff}}\approx1$). Figure~\ref{fig:branches} confirms
Remark~\ref{rem:when-rank-helps}: for near-collinear features all
three quantities coincide at every $r$ (restricting the rank buys
nothing); for well-spread features the joint-budget value grows like
$\sqrt r$ (fitted log--log slope $0.445$) before saturating, and the
saturation point tracks the effective rank, while the per-factor
value is a constant far below the ceiling. The rank constraint is
informative only under the joint budget and only for well-spread
features; under LoRA's own per-factor budget the complexity is flat
in $r$ throughout.
 
\begin{figure*}[t]
\centering
\includegraphics[width=\textwidth]{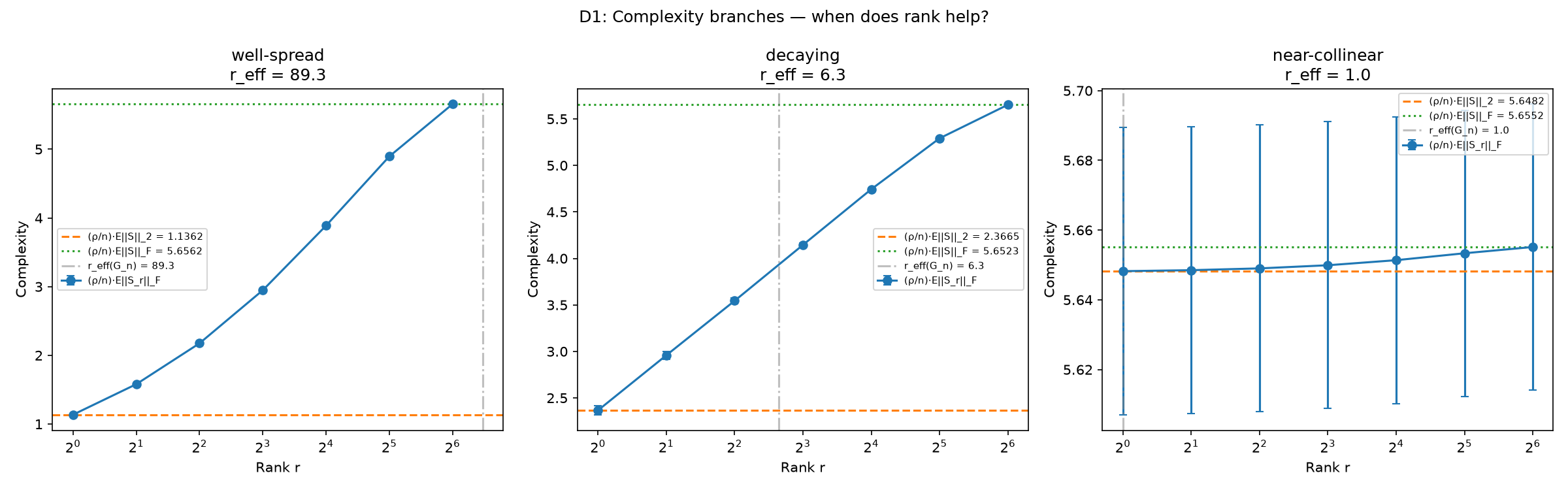}
\caption{Exact complexity values versus $r$ for three feature
spectra. Horizontal lines: per-factor value
$\tfrac\rho n\mathbb E\|S\|_2$ (lower) and Frobenius ceiling
$\tfrac\rho n\mathbb E\|S\|_F$ (upper); curve: joint-budget value
$\tfrac\rho n\mathbb E\|S_r\|_F$. Rank helps only under the joint
budget and only when the data's effective rank is large
(Lemma~\ref{lem:matrix-rad}, Remark~\ref{rem:when-rank-helps}).}
\label{fig:branches}
\end{figure*}
 
\subsection{Implicit Nuclear-Norm Bias of Training}
\label{sec:exp-implicit}
 
Finally we probe the open question of
Section~\ref{sec:conclusion}: does per-factor weight decay select
low-nuclear-norm solutions? We train unconstrained LoRA (no
projection) to recover a rank-$4$ target
$W^\ast\in\mathbb R^{256\times256}$ with singular values
$(1,0.7,0.5,0.35)$, hence $\|W^\ast\|_*=2.55$, from $2000$ samples,
sweeping $r\in\{2,\dots,64\}$ and weight decay in
$\{0,10^{-3},10^{-2},10^{-1}\}$ ($3000$ AdamW steps; the $r=2$ runs
reach the Eckart--Young floor to within $2\%$, confirming
optimization quality). Two findings. First, the learned update's
nuclear norm stays within $13\%$ of $\|W^\ast\|_*$ and its effective
rank stays at $\approx4$ even at nominal rank $64$---the optimizer
simply does not spend the available rank. Second, and contrary to the
natural conjecture, this happens \emph{with or without weight decay}:
at these horizons the drift $\|BA\|_*$ from $r=4$ to $r=64$ is
$12.7\%$ at zero weight decay versus $11.9$--$13.3\%$ across the
weight-decay grid, an indistinguishable difference. The
low-nuclear-norm behavior is thus an implicit bias of the
optimization itself in this setting; whether weight decay becomes the
operative mechanism at longer horizons or in deep nonlinear models
remains open, and the question posed in
Section~\ref{sec:conclusion} is sharpened rather than settled: the
quantity training implicitly controls appears to be the nuclear norm,
but weight decay is not demonstrably the cause.
 
\paragraph*{Summary}
Each prediction of the theory that can be instantiated at this scale
is observed: the worst case over $\mathcal A_r$ is rank-one and flat
in $r$ (Lemma~\ref{lem:extremal}); the realized ERM gap is rank-free
exactly where the nuclear budget binds, with level linear in $\rho$
(Propositions~\ref{prop:collapse}--\ref{prop:matrix-collapse}), and
rank matters only through affordability below that threshold; the
threshold-rank scaling laws hold at their matched rates exactly when
the directional alignment condition is enforced
(Theorem~\ref{thm:minrank}, Proposition~\ref{prop:lower}); and the
joint-budget complexity is rank-sensitive only for well-spread
features (Lemma~\ref{lem:matrix-rad}). None of this constitutes
evidence about deep-network fine-tuning; it confirms that the
regimes the theory delineates are real and appear where predicted.

\section{Discussion and Conclusion}
\label{sec:conclusion}

Returning to the three questions:

\textbf{Q1 (statistical).} Under LoRA's own parameterization the adaptation rank buys no complexity reduction in any of the functionals we compute. The free-readout class is literally the same set of functions for every rank (Proposition~\ref{prop:collapse}), and the collapse persists when one bounds the matrix-valued layer directly, where the complexity is $\rho\,\mathbb E\|S\|_2/n$ exactly (Proposition~\ref{prop:matrix-collapse}). No scalar single-layer linear-readout formalization considered here exhibits rank-dependence, with or without the frozen weight: each of the two natural modifications---a joint Frobenius budget on the product, a fixed readout---is insufficient alone, and even both together leave the scalar complexity rank-free (Proposition~\ref{prop:necessity}, Remark~\ref{rem:frozen-scalar}); rank enters exactly at the matrix level, through the joint-budget complexity $\rho\,\mathbb E\|S_r\|_F/n$. That quantity is rank-sensitive only in its data-dependent form: its operator-norm upper bound can improve on the rank-free Frobenius branch when $q_r\log(m+d)$ is small relative to the effective rank of the data covariance, with $q_r=\min\{r,m,d\}=r$ in the practical regime (Remark~\ref{rem:when-rank-helps}), while the worst-case simplification $\rho X_{\max}\sqrt{m/n}$ is provably rank-free (Lemma~\ref{lem:matrix-rad}(ii)). The $\sqrt{q_r}$ in the operator branch is the sharp worst-case nuclear--Frobenius comparison factor between the two constraint sets (Remark~\ref{rem:sqrt-r-is-the-gap}).

\textbf{Q2 (spectral).} Rank governs how large a nuclear budget is required to cancel a given spectral block of $W_0$. We characterize the spectral threshold rank of the top-block cancellation certificate with upper and lower bounds that match in order under two-sided spectral decay, comparable directional constants ($c_1\asymp\|\Delta\Sigma\|_2$), and away from rank saturation (Theorem~\ref{thm:minrank}, Proposition~\ref{prop:lower}), under an alignment assumption verified uniformly along the cancellation family in a concrete Gaussian model with range-adapted readouts (Assumption~\ref{ass:kappa}, Example~\ref{ex:gaussian}); the budget-feasible minimal certified rank, whenever it exists, coincides with the spectral threshold (Remark~\ref{rem:rcert}). Two generalization routes follow: a Wasserstein-transport bound (Theorem~\ref{thm:gen}), and a spectral-cancellation bound, in population and empirical form, in which the discrepancy term is controlled by $\kappa\tau_r$ and the certificate class's own complexity shrinks with the cancelled rank (Corollary~\ref{cor:spectral-route}). The transport bound separates the two roles of the frozen weight: absent from the complexity term, unavoidable in the transport term.

\textbf{Q3 (geometric).} Whenever an admissible adapted weight is singular---automatic when $m<d$---the divergence-defined cone has ambient-scale radius (Proposition~\ref{prop:trivial-radius}), and even the pushforward-realizable restriction obeys a rank-independent universal radius bound under a fixed budget---sharp under the mild budget condition $s_d(W_0)\le\rho$ of Proposition~\ref{prop:tight}, with a rank-one Gaussian witness on whose instance the exact radius is attained for every rank (Theorem~\ref{thm:cone}).

The practical reading is that rank should not be selected as a capacity control. It should be selected against the spectral decay of the pre-trained weight and the norm budget one is willing to spend---with the effective rank of the feature distribution determining whether any statistical benefit is available at all. A natural next question, which our framework poses but does not settle, is whether optimization dynamics under per-factor weight decay implicitly select low-nuclear-norm rather than low-rank solutions; if so, the results here suggest the nuclear norm, not $r$, is the quantity worth tuning.
Our experiments (Section~\ref{sec:exp-implicit}) show that in the
linear setting the trained update indeed keeps a near-minimal nuclear
norm and effective rank regardless of the nominal rank---but it does
so with or without weight decay, so the selection appears to be an
implicit bias of the optimization dynamics themselves, leaving the
mechanism question open in a sharpened form.
\bibliographystyle{IEEEtran}
\bibliography{reference}
\newpage
\appendices

\section{Auxiliary Results}
\label{app:aux}

\begin{lemma}[Norm Submultiplicativity]
\label{lem:submult}
For $B\in\mathbb R^{m\times r}$, $A\in\mathbb R^{r\times d}$: $\|BA\|_F\le\|BA\|_*\le\|B\|_F\|A\|_F$.
\end{lemma}

\begin{proof}
The first inequality holds for any matrix, since $\|M\|_F^2=\sum_is_i(M)^2\le(\sum_is_i(M))^2$. For the second, by duality of $\|\cdot\|_*$ and $\|\cdot\|_2$ and then Cauchy--Schwarz,
\begin{equation}
\begin{aligned}
\|BA\|_* &= \max_{\|Z\|_2\le1}\langle BA,Z\rangle_F = \max_{\|Z\|_2\le1}\langle A, B^TZ\rangle_F\\
&\le \max_{\|Z\|_2\le1}\|A\|_F\|B^TZ\|_F \le \|A\|_F\|B\|_F,
\end{aligned}
\end{equation}
using $\|B^TZ\|_F\le\|B\|_F\|Z\|_2$.
\end{proof}

\begin{lemma}[Vector-Valued Translation Invariance]
\label{lem:translation}
Let $F=\{x\mapsto W_0x+\Delta Wx : \Delta W\in\mathcal D\}$ for a \emph{fixed} $W_0$ and any admissible set $\mathcal D$. Then $\mathfrak R_n^{\mathrm{vec}}(F)=\mathfrak R_n^{\mathrm{vec}}(\{x\mapsto\Delta Wx:\Delta W\in\mathcal D\})$ exactly: the frozen matrix contributes nothing.
\end{lemma}

\begin{proof}
For each realization of $\{\sigma_i\}$, since $W_0x_i$ does not depend on $\Delta W$,
\begin{equation}
\begin{aligned}
\sup_{\Delta W}\sum_i
 \langle \sigma_i, W_0x_i+\Delta W x_i\rangle
&= \sum_i \langle \sigma_i, W_0x_i\rangle \\
&\quad + \sup_{\Delta W}
 \sum_i \langle \sigma_i,\Delta W x_i\rangle .
\end{aligned}
\end{equation}

Taking $\mathbb E_\sigma$ and using $\mathbb E[\sigma_i]=0$ annihilates the first summand exactly.
\end{proof}

\begin{remark}
This is not in tension with Proposition~\ref{prop:collapse}(ii), where a $\|W_0\|_2$ term survives. There the readout $\mathbf w$ is a free variable coupling the frozen and adaptive terms, so the two cannot be separated and the $W_0$ contribution is real---indeed Proposition~\ref{prop:collapse}(ii) is an identity. Here $\Delta W$ is the only free parameter and $W_0$ is a pure additive constant, so the decomposition is exact and its contribution vanishes.
\end{remark}

\begin{lemma}[Contraction to the Scalar Class]
\label{lem:contraction}
Let $S:=\sum_{i=1}^n\sigma_ix_i^T$ with $\sigma_i\in\{\pm1\}^m$ i.i.d.\ Rademacher vectors. For a constraint set $\mathcal D_r\in\{\mathcal A_r,\ \mathcal B_{r,\rho}\}$ define
\begin{equation}
\mathfrak R_n^{\mathrm{mat}}(\mathcal D_r) := \frac1n\,\mathbb E_\sigma\Big[\sup_{\Delta W\in\mathcal D_r}\langle\Delta W,S\rangle_F\Big],
\end{equation}
so that $\mathfrak R_n^{\mathrm{mat}}(\mathcal A_r)=\tfrac\rho n\mathbb E\|S\|_2$ (Proposition~\ref{prop:matrix-collapse}) and $\mathfrak R_n^{\mathrm{mat}}(\mathcal B_{r,\rho})=\tfrac\rho n\mathbb E\|S_r\|_F=\mathfrak R_n^{\mathrm{mat}}$ (Lemma~\ref{lem:matrix-rad}). With $\psi$ fixed and $L_\psi$-Lipschitz as in Definition~\ref{def:fixed-readout} and $\mathcal H_r^\psi$ built on $\mathcal D_r$,
\begin{equation}
\mathfrak{R}_n(\mathcal{H}_r^\psi) \le \sqrt2\, L_\psi\, \mathfrak{R}_n^{\mathrm{mat}}(\mathcal D_r).
\end{equation}
\end{lemma}

\begin{proof}
By Lemma~\ref{lem:translation} and the trace identity $\sum_i\langle\sigma_i,\Delta Wx_i\rangle=\langle\Delta W,S\rangle_F$, the class $\{x\mapsto(W_0+\Delta W)x:\Delta W\in\mathcal D_r\}$ has vector-valued complexity $\mathfrak R_n^{\mathrm{mat}}(\mathcal D_r)$. Maurer's vector-contraction inequality \cite{maurer2016vector}---which requires the Lipschitz map to be fixed rather than jointly optimized, and carries the explicit constant $\sqrt2$---gives $\mathfrak R_n(\psi\circ F)\le\sqrt2L_\psi\mathfrak R_n^{\mathrm{vec}}(F)$. If one prefers a contraction statement for maps vanishing at zero, replace $\psi$ by $\psi-\psi(0)$: this changes neither $L_\psi$ nor $\mathfrak R_n(\mathcal H_r^\psi)$, since the constant shift adds the $h$-independent term $\psi(0)\sum_i\sigma_i$, of zero expectation, inside the supremum.
\end{proof}

\begin{lemma}[Monotonicity of the Tail Profile]
\label{lem:monotone}
$\tau_0\ge\tau_1\ge\tau_2\ge\cdots$, and $\tau_r=0$ for $r\ge\mathrm{rank}(W_0)$.
\end{lemma}

\begin{lemma}[Spectral Cancellation is Realizable]
\label{lem:spectral}
Let $W_0=(W_0)_{\le r}+(W_0)_{>r}$ be the SVD split of $W_0$. If $\rho\ge\|(W_0)_{\le r}\|_*=\sum_{i\le r}s_i(W_0)$, then $\Delta W:=-(W_0)_{\le r}$ belongs to $\mathcal A_r$, and $W:=W_0+\Delta W=(W_0)_{>r}$, so that $\|W\Delta\Sigma W^T\|_2=\tau_r$ for every target.
\end{lemma}

\begin{proof}
$\Delta W$ has rank at most $r$ and $\|\Delta W\|_*=\sum_{i\le r}s_i(W_0)\le\rho$, so $\Delta W\in\mathcal A_r$ by Lemma~\ref{lem:factorization}. The identity $W=(W_0)_{>r}$ is immediate.
\end{proof}

\begin{proposition}[Displacement of the Cancellation Pushforward; $m=d$]
\label{prop:cancel-dist}
With $\Delta W=-(W_0)_{\le r}$ as in Lemma~\ref{lem:spectral}, $D_S$ of finite second moment, and $D_T=(I+\Delta W)_\#D_S$,
\begin{equation}
W_2(D_S,D_T) \le \big(\mathrm{Tr}\big[(W_0)_{\le r}\,\Sigma_S\,(W_0)_{\le r}^T\big]\big)^{1/2}.
\end{equation}
\end{proposition}

\begin{proof}
The linear pushforward is a feasible coupling, so $W_2^2\le\mathbb E_{D_S}\|\Delta Wx\|_2^2=\mathrm{Tr}(\Delta W\Sigma_S\Delta W^T)$; substitute $\Delta W=-(W_0)_{\le r}$. This is generally only an upper bound: $W_2$ is an infimum over \emph{all} couplings, and the linear map need not be the optimal-transport map (equality would additionally require, e.g., Gaussian marginals with commuting covariances).
\end{proof}

\begin{example}[Gaussian instantiation of Assumption~\ref{ass:kappa}]
\label{ex:gaussian}
Let $D_S,D_T$ be centered Gaussians on $\mathbb R^d$ with covariances $\Sigma_S,\Sigma_T$ and $\Sigma_T\succ0$ (Assumption~\ref{ass:bddfeat} is not in force here). If $W_0=0$ every cancellation weight is $0$, every induced class is trivial, and the assumption holds vacuously with any $\kappa>0$; so assume $W_0\ne0$, write $K:=\mathrm{rank}(W_0)\ge1$, and let $s_K(W_0)>0$ be the smallest nonzero singular value of $W_0$. For each weight $W$ take the \emph{range-adapted} (output-space) bounded linear readouts
\begin{equation}
\Psi_W=\big\{z\mapsto\mathbf w^Tz\ :\ \|\mathbf w\|_2\le1,\ \mathbf w\in\mathrm{range}(W)\big\},
\end{equation}
where $\mathrm{range}(W)\subseteq\mathbb R^m$ is the column space of $W$. We emphasize that this is an explicitly weight-dependent readout family $\{\Psi_W\}$, as permitted by Definition~\ref{def:induced}: the example verifies the assumption for this restricted model, not for an arbitrary readout class fixed independently of $W$. With this choice, Assumption~\ref{ass:kappa} holds for the entire cancellation family $\{W=(W_0)_{>r'} : r'\in\{0,1,\dots,K-1\}\}$ with the single constant
\begin{equation}
\kappa=\tfrac2c\sqrt{2/\pi},\qquad c:=\lambda_{\min}(\Sigma_T)^{1/2}\,s_K(W_0).
\end{equation}
Indeed, fix such a $W$. Its nonzero singular values are $s_{r'+1}(W_0)\ge\cdots\ge s_K(W_0)$, so expanding any $u\in\mathrm{range}(W)$ in the left singular vectors of $W$ gives $\|W^Tu\|_2\ge s_K(W_0)\|u\|_2$, whence
\begin{equation}
\beta:=u^TW\Sigma_TW^Tu\ \ge\ \lambda_{\min}(\Sigma_T)\|W^Tu\|_2^2\ \ge\ c^2\|u\|_2^2 .
\end{equation}
For $h=\mathbf w^TWx$, $h'=\mathbf w'^TWx$ with $\mathbf w,\mathbf w'$ in the unit ball of $\mathrm{range}(W)$ and $u:=\mathbf w-\mathbf w'\in\mathrm{range}(W)$ (so $\|u\|_2\le2$), the Gaussian identity $\mathbb E|Z|=\sqrt{2/\pi}(\mathrm{Var}\,Z)^{1/2}$ gives
\begin{equation}
\mathbb E_{D}|h-h'| = \sqrt{\tfrac2\pi}\big(u^TW\Sigma_DW^Tu\big)^{1/2},\qquad D\in\{D_S,D_T\}.
\end{equation}
For $u\ne0$, writing $a:=u^TW\Sigma_SW^Tu$,

\begin{equation}
\begin{aligned}
\bigl|\sqrt{a}-\sqrt{\beta}\bigr|
&= \frac{|a-\beta|}{\sqrt{a}+\sqrt{\beta}} \\
&\le
\frac{\|u\|_2^2\,\|W\Delta\Sigma W^\top\|_2}
     {c\|u\|_2} \\
&\le \frac{2}{c}\,
\|W\Delta\Sigma W^\top\|_2 .
\end{aligned}
\end{equation}

and $u=0$ is trivial. Taking the supremum over $\mathbf w,\mathbf w'$ gives the claim; for $r'\ge K$ one has $W=0$, every member of $\mathcal H_W$ is identically zero, and both sides vanish. Two remarks. First, the constant is uniform over the family precisely because the tail's smallest nonzero singular value equals $s_K(W_0)$ for every $r'<K$. Second, the restriction of the readouts to $\mathrm{range}(W)$ is essential: cancellation weights are rank-deficient, so no bound of the form $\lambda_{\min}(W\Sigma_TW^T)\ge c^2>0$ over all of $\mathbb R^m$ can hold for them, and an unrestricted-readout version of this example would fail exactly at the weights where the assumption is used (cf.\ Remark~\ref{rem:kappa-scope}).
\end{example}

\begin{proposition}[Tightness Witness for Theorem~\ref{thm:cone}(i)]
\label{prop:tight}
Let $m=d$ and $D_S=\mathcal N(0,I_d)$, so $\lambda_{\max}(\Sigma_S)=1$ (Assumption~\ref{ass:bddfeat} is not in force here), and suppose the budget covers the weakest input direction of the frozen weight:
\begin{equation}
\label{eq:witness-budget}
s_d(W_0)\ \le\ \rho
\end{equation}
(automatic, for every budget, when $W_0$ is singular). Let $u$ be a unit right singular vector of $W_0$ attaining $\|W_0u\|_2=s_d(W_0)$, let $e_1\in\mathbb R^r$ be the first standard basis vector, and set $B:=B_B\,ue_1^T\in\mathbb R^{m\times r}$, $A:=B_A\,e_1u^T\in\mathbb R^{r\times d}$, so that $\|B\|_F=B_B$, $\|A\|_F=B_A$ and $\Delta W:=BA=\rho\,uu^T$ has rank one, and put $D_T:=(I+\Delta W)_\#D_S$. Then for every $\epsilon\ge0$ and every $r\ge1$:
\begin{enumerate}
\item[(i)] \emph{(Membership.)} $D_T\in\mathcal C_{r,\mathrm{two}}^{\mathrm{push}}(\epsilon)$.
\item[(ii)] \emph{(Extremality.)} $W_2(D_S,D_T)=\rho$, matching the bound of Theorem~\ref{thm:cone}(i) exactly.
\end{enumerate}
As anticipated in Remark~\ref{rem:two-witnesses}, the update certifying membership in part (i) is different from the update generating the pushforward.
\end{proposition}

\begin{proof}
(i) $D_T$ is by construction a pushforward by an admissible update, since $\Delta W=\rho uu^T\in\mathcal A_1\subseteq\mathcal A_r$. It remains to verify $D_T\in\mathcal C_r(W_0,D_S;\epsilon)$, i.e.\ to exhibit \emph{some} admissible update whose induced class has discrepancy at most $\epsilon$. Take $\Delta W':=-W_0uu^T=-(W_0u)u^T$. It has rank at most one (it vanishes when $W_0u=0$) and, by \eqref{eq:witness-budget}, $\|\Delta W'\|_*=\|W_0u\|_2\|u\|_2=s_d(W_0)\le\rho$, so $\Delta W'\in\mathcal A_1\subseteq\mathcal A_r$ by Lemma~\ref{lem:factorization} (and $\|\Delta W'\|_F\le\rho$, so also $\Delta W'\in\mathcal B_{r,\rho}$). The weight $W':=W_0+\Delta W'=W_0(I-uu^T)$ satisfies $W'u=0$. Since $D_T$ is the law of $x+\rho(u^Tx)u$ with $x\sim D_S$,
\begin{equation}
W'\big(x+\rho(u^Tx)u\big)=W'x,
\end{equation}
so $W'x$ has the same law under $D_S$ and $D_T$, and therefore $\mathbb E_{D_S}|h-h'|=\mathbb E_{D_T}|h-h'|$ for all $h,h'\in\mathcal H_{W'}$, whatever the readout set. Hence $\tilde d_{\mathcal H_{W'}}(D_S,D_T)=0\le\epsilon$, and $D_T\in\mathcal C_{r,\mathrm{two}}^{\mathrm{push}}(\epsilon)$.

(ii) The upper bound $W_2\le\rho$ is Theorem~\ref{thm:cone}(i). For the lower bound, let $(X,Y)$ be any coupling of the two laws. The projections $(u^TX,u^TY)$ couple $\mathcal N(0,1)$ with $\mathcal N(0,(1+\rho)^2)$ (isotropy of $D_S$ makes this valid for any unit $u$, in particular the chosen singular direction), and $\mathbb E\|X-Y\|_2^2\ge\mathbb E(u^TX-u^TY)^2$. For real variables with second moments $1$ and $(1+\rho)^2$, Cauchy--Schwarz gives $\mathbb E(u^TX-u^TY)^2\ge 1+(1+\rho)^2-2(1+\rho)=\rho^2$. Hence $W_2\ge\rho$, with equality. Since both the pushforward witness $\Delta W$ and the membership certificate $\Delta W'$ have rank at most one, enlarging $r$ cannot enlarge the radius and restricting to $r=1$ cannot shrink it: under \eqref{eq:witness-budget} the radius is genuinely rank-independent for this instance, not merely rank-independently bounded.
\end{proof}

\section{Secondary Results and Extended Remarks}
\label{app:remarks}

This appendix collects, in the order of the main text, secondary results and extended discussion supporting the results of Sections~\ref{sec:framework}--\ref{sec:cone}.

\begin{remark}\label{rem:disc-position}
$\tilde d_{\mathcal H}$ is a discrepancy distance in the sense of Mansour, Mohri, and Rostamizadeh~\cite{mansour2009domain}, instantiated with the absolute loss. It is a continuous-valued analogue of, but \emph{not} identical to, the binary $\mathcal H\Delta\mathcal H$-divergence of Ben-David et al.~\cite{ben2006analysis}: the latter is a supremum over probabilities of disagreement events for thresholded hypotheses, whereas $\tilde d_{\mathcal H}$ is a supremum over expected absolute score differences. Every result below is stated and proved for $\tilde d_{\mathcal H}$ natively; no transfer from the binary theory is invoked anywhere, and in particular no margin condition relating hard and soft disagreement is needed.
\end{remark}

\begin{remark}[On the $\alpha/r$ scaling used in practice]
\label{rem:alpha-scaling}
Standard implementations parameterize the update as $\Delta W=\tfrac{\alpha_{\mathrm{LoRA}}}rBA$ for a fixed hyperparameter $\alpha_{\mathrm{LoRA}}>0$ (we reserve the plain symbol $\alpha$ for the spectral decay exponent of Section~\ref{sec:spectral}). The reachable set is then $\{\Delta W:\mathrm{rank}(\Delta W)\le r,\ \|\Delta W\|_*\le\rho_r\}$ with \emph{effective} nuclear radius $\rho_r:=\tfrac{\alpha_{\mathrm{LoRA}}}rB_BB_A$: all extremal formulas below remain valid after replacing $\rho$ by $\rho_r$, while the individual reachable updates depend on both the rank cap and $\rho_r$. Under that convention, the rank-one support and Frobenius extremals over the per-factor set change only through the scalar radius $\rho_r$---an instance of, not an exception to, the affordability reading above---while the joint-Frobenius matrix complexity of Lemma~\ref{lem:matrix-rad} retains its additional dependence on the truncated signal $S_r$. Throughout the paper $\rho$ always denotes the budget on the actual product $\Delta W$.
\end{remark}

\begin{remark}\label{rem:collapse-mechanism}
Part (i) is strictly stronger than (ii): it forecloses \emph{every} statistical separation between ranks in this formalization, since covering numbers, chaining functionals, and Rademacher averages all depend on a class only as a set of functions. The mechanism, visible in the proof, is that a free readout $\mathbf w$ can absorb the entire relevant action of any admissible $\Delta W$ into the rank-one surrogate $\|\mathbf w\|_2^{-2}\,\mathbf w(\mathbf w^T\Delta W)$, which lies in $\mathcal A_1$ by Lemma~\ref{lem:factorization}. Note also that (ii) is an identity rather than an inequality, so the $\|W_0\|_2$ term in \eqref{eq:collapsed} is genuine and not an artifact of decoupling the two suprema.
\end{remark}

\begin{proposition}[Necessity Without Sufficiency for the Scalar Classes]
\label{prop:necessity}
Let $v=\sum_i\sigma_ix_i$ and $S=\sum_i\sigma_ix_i^T$ be as in Section~\ref{sec:complexity}.
\begin{enumerate}
\item[(i)] \emph{(Fixing the readout alone is not enough.)} For the class $\{x\mapsto\mathbf w^T\Delta Wx:\Delta W\in\mathcal A_r\}$ with $\mathbf w$ fixed, $\|\mathbf w\|_2\le1$, the empirical Rademacher complexity equals $\tfrac{\rho}{n}\|\mathbf w\|_2\,\mathbb E\|v\|_2$, with no $r$.
\item[(ii)] \emph{(Relaxing the budget alone is not enough.)} For the free-readout class $\{x\mapsto\mathbf w^T\Delta Wx : \|\mathbf w\|_2\le1,\ \Delta W\in\mathcal B_{r,\rho}\}$, the empirical Rademacher complexity equals $\tfrac{\rho}{n}\mathbb E\|v\|_2$, with no $r$.
\item[(iii)] \emph{(Both changes together are still not enough.)} For the class $\{x\mapsto\mathbf w^T\Delta Wx:\Delta W\in\mathcal B_{r,\rho}\}$ with $\mathbf w$ fixed, $\|\mathbf w\|_2\le1$, the empirical Rademacher complexity equals $\tfrac{\rho}{n}\|\mathbf w\|_2\,\mathbb E\|v\|_2$, with no $r$.
\end{enumerate}
Consequently no scalar single-layer \emph{linear}-readout formalization considered here---free or fixed readout, per-factor or joint budget---exhibits rank-dependence: the two changes are necessary for rank-dependence to be possible, and by (iii) they are not sufficient at the scalar linear level. Rank enters exactly at the matrix-valued level, through the value $\rho\,\mathbb E\|S_r\|_F/n$ of Lemma~\ref{lem:matrix-rad}, with no readout at all; for a fixed \emph{nonlinear} readout, over either constraint set, our contraction bounds are upper bounds rather than identities, so rank-dependence there is neither established nor excluded.
\end{proposition}

\begin{remark}[Adding the frozen weight changes nothing]
\label{rem:frozen-scalar}
The classes in Proposition~\ref{prop:necessity} are adaptation-only; the conclusions persist verbatim for the full-weight classes $x\mapsto\mathbf w^T(W_0+\Delta W)x$. For a \emph{fixed} readout, the frozen term contributes $\sum_i\sigma_i\mathbf w^TW_0x_i$, which does not depend on $\Delta W$ and has zero $\sigma$-expectation, so the complexities in cases (i) and (iii) are unchanged. For the \emph{free} readout over $\mathcal B_{r,\rho}$, the computation of Proposition~\ref{prop:collapse}(ii) goes through unchanged---the supremum of $\langle\Delta W,\mathbf wv^T\rangle_F$ over $\mathcal B_{r,\rho}$ equals $\rho\|\mathbf w\|_2\|v\|_2$, the same value as over $\mathcal A_r$, by the rank-one-signal argument of part (iii)---so the exact full-weight complexity is again $\tfrac1n\mathbb E_\sigma[\|W_0v\|_2+\rho\|v\|_2]$, with no $r$.
\end{remark}

\begin{remark}[The $\sqrt{q_r}$ is the worst-case nuclear--Frobenius gap]
\label{rem:sqrt-r-is-the-gap}
Comparing Proposition~\ref{prop:matrix-collapse} with \eqref{eq:datadep}, the rank-dependent branch $\tfrac\rho n\sqrt{q_r}\,\mathbb E\|S\|_2$ is exactly the collapse value $\tfrac\rho n\mathbb E\|S\|_2$ with budget inflated by $\sqrt{q_r}$---precisely the right-hand inclusion in \eqref{eq:sandwich} (recall $q_r=r$ whenever $r\le\min(m,d)$). Here $\sqrt{q_r}$ is the sharp \emph{worst-case} nuclear--Frobenius comparison factor over $\mathcal B_{r,\rho}$; for an individual matrix $M$ the factor is $\sqrt{\mathrm{rank}(M)}\le\sqrt{q_r}$. In other words, the apparent gain from restricting the rank under $\mathcal B_{r,\rho}$ is the mirror image of the loss incurred by relaxing the nuclear budget to a Frobenius one. Under the true LoRA constraint set $\mathcal A_r$ one already has the better, $r$-free value $\tfrac\rho n\mathbb E\|S\|_2$, and nothing is gained by shrinking $r$.
\end{remark}

\begin{remark}[When the rank constraint helps within $\mathcal B_{r,\rho}$]
\label{rem:when-rank-helps}
Comparing the two branches of \eqref{eq:datadep} via \eqref{eq:frob-branch}--\eqref{eq:op-branch} and ignoring the lower-order term, the rank-dependent branch beats the rank-free Frobenius branch when
\begin{equation}
2\,q_r\log(m+d)\cdot\max\big(\mathrm{tr}(G_n),\,m\|G_n\|_2\big) \lesssim m\,\mathrm{tr}(G_n),
\end{equation}
i.e.\ when $q_r\log(m+d)\lesssim\min\big(m,\ r_{\mathrm{eff}}(G_n)\big)$ (with $q_r=r$ in the practical regime $r\le\min(m,d)$). We stress the epistemic status of this comparison: both branches are upper bounds on the exact value $\tfrac\rho n\mathbb E\|S_r\|_F$, and the operator branch additionally discards the lower-order Bernstein term, so the condition identifies a \emph{sufficient regime for the rank-dependent bound to improve}, not an exact characterization of the complexity. Two consequences. In the spectrally degenerate case (all $x_i$ collinear, $r_{\mathrm{eff}}=1$) the rank constraint yields \emph{no} improvement. For well-spread features ($r_{\mathrm{eff}}\asymp\min(n,d)$) it yields savings of order $\sqrt{r_{\mathrm{eff}}/(q_r\log(m+d))}$ once $q_r\log(m+d)\ll\min(m,n,d)$. And in the worst case allowed by Assumption~\ref{ass:bddfeat} the rank branch never wins at all: by Lemma~\ref{lem:matrix-rad}(ii) the worst-case bound is rank-free.
\end{remark}

\begin{remark}[On the $\sqrt m$ factor]
\label{rem:sqrtm}
The bound carries an explicit $\sqrt m$ relative to a scalar-noise heuristic that treats $\sigma_i$ as though its size were independent of the output dimension. Since $\sigma_i\in\{\pm1\}^m$ has deterministic norm $\sqrt m$, both the matrix-Bernstein variance proxy and the almost-sure norm bound pick up factors of $m$ and $\sqrt m$ respectively, and these propagate to the rate. We also note that all absolute constants here and downstream are conditional on the precise formulations of the two external inequalities we invoke---the expectation form of the matrix Bernstein inequality \cite{tropp2012user} in \eqref{eq:op-branch} and Maurer's vector-contraction inequality \cite{maurer2016vector} with constant $\sqrt2$ in Lemma~\ref{lem:contraction}; alternative statements of these inequalities carry slightly different absolute constants or logarithmic factors.
\end{remark}

Two auxiliary facts, both in Appendix~\ref{app:aux}, transfer Lemma~\ref{lem:matrix-rad} to $\mathcal H_r^\psi$: the frozen weight contributes \emph{exactly} nothing to the vector-valued complexity (Lemma~\ref{lem:translation}), and Maurer's vector-contraction inequality \cite{maurer2016vector} passes through the fixed readout at cost $\sqrt2L_\psi$ (Lemma~\ref{lem:contraction}).

\begin{corollary}[Complexity Bounds for the Fixed-Readout Classes]
\label{cor:rank-dep}
Write $\mathcal H_{r,\mathcal D_r}^\psi$ for the class of Definition~\ref{def:fixed-readout} built on the constraint set $\mathcal D_r$. Under Assumption~\ref{ass:bddfeat} and for $n\ge\log(m+d)$,
\begin{align}
\mathfrak{R}_n(\mathcal{H}_{r,\mathcal A_r}^\psi)
&\le
\frac{\sqrt{2}\,L_\psi\,\rho}{n}\,
\mathbb{E}\|S\|_2
\notag\\
&\le
\sqrt{2}\,L_\psi\,\rho\,X_{\max}
\sqrt{\frac{m}{n}},
\\[0.5em]
\mathfrak{R}_n(\mathcal{H}_{r,\mathcal B_{r,\rho}}^\psi)
&\le
\frac{\sqrt{2}\,L_\psi\,\rho}{n}
\min\!\left\{
\sqrt{q_r}\,\mathbb{E}\|S\|_2,\,
\mathbb{E}\|S\|_F
\right\}
\notag\\
&\le
\sqrt{2}\,L_\psi\,\rho\,X_{\max}
\sqrt{\frac{m}{n}}.
\end{align}

where in each line the final inequality (the rank-free worst case) follows from $\|S\|_2\le\|S\|_F$, $\|S_r\|_F\le\|S\|_F$, and \eqref{eq:frob-branch}. Rank enters only through the data-dependent minimum in the $\mathcal B_{r,\rho}$ line, whose operator-norm branch can improve on the Frobenius branch under the effective-rank condition of Remark~\ref{rem:when-rank-helps}; the $\mathcal A_r$ line is rank-free.
\end{corollary}

\begin{remark}[Scope]
\label{rem:linear-readout}
Corollary~\ref{cor:rank-dep} is delimited by three caveats, and we state them plainly. First, the rank-sensitive minimum concerns $\mathcal B_{r,\rho}$ only; the $\mathcal A_r$ line is rank-free, its underlying vector-valued complexity being exactly $\rho\,\mathbb E\|S\|_2/n$ (Proposition~\ref{prop:matrix-collapse}). Second, when $\psi$ is itself \emph{linear}, $\psi(z)=\mathbf w_0^Tz$ with $\|\mathbf w_0\|_2\le1$ fixed, the class is scalar-linear in $x$ and the direct bound $\mathfrak R_n(\mathcal H_r^\psi)\le\rho X_{\max}/\sqrt n$ (no $W_0$ term, by Lemma~\ref{lem:translation}) beats the worst case of Corollary~\ref{cor:rank-dep} by a factor of $\sqrt{2m}$; the vector-contraction route, and with it the $m$ dependence, is informative only for genuinely nonlinear $\psi$. Third, even within $\mathcal B_{r,\rho}$ the rank branch improves on the rank-free branch only under Remark~\ref{rem:when-rank-helps}'s effective-rank condition---never under worst-case features.
\end{remark}

\begin{remark}[Scope of Assumption~\ref{ass:kappa}]
\label{rem:kappa-scope}
The assumption is deliberately restricted to the cancellation weights, which are the only adapted weights at which it is applied (in Theorem~\ref{thm:minrank}, Corollary~\ref{cor:spectral-route}, and Theorem~\ref{thm:cone}(ii)). Requiring it for \emph{all} admissible $W=W_0+\Delta W$, $\Delta W\in\mathcal A_r$, would be strictly stronger and is typically unavailable: cancellation weights are rank-deficient (rank $m-r'$ when $W_0$ has full row rank and $r'\in\{0,1,\dots,m\}$; identically zero for larger $r'$), so pointwise nondegeneracy conditions of the form $\lambda_{\min}(W\Sigma_TW^T)\ge c^2>0$ on all of $\mathbb R^m$ necessarily fail on them. The range-adapted formulation of Example~\ref{ex:gaussian} is what restores a constant $\kappa$ that is uniform over the whole family.
\end{remark}

\begin{remark}[Spectral threshold versus budget-feasible certified rank]
\label{rem:rcert}
$r_{\mathrm{spec}}(\epsilon)$ is a spectral quantity; it is a rank at which cancellation \emph{would} certify the tolerance, not automatically a rank at which cancellation is \emph{affordable}. The two conditions move in opposite directions: $\{r:\tau_r\le\epsilon/\kappa\}$ is upward closed (Lemma~\ref{lem:monotone}), while the budget condition $\{r:\sum_{i\le r}s_i(W_0)\le\rho\}$ is \emph{downward} closed, so their intersection---the set of budget-feasible certifying ranks---is not an upward-closed threshold set but an \emph{interval}
\begin{equation}
\begin{aligned}
\bigl[r_{\mathrm{spec}}(\epsilon),\,r_{\mathrm{bud}}(\rho)\bigr],
\qquad
r_{\mathrm{bud}}(\rho)
&:= \max\Bigl\{
r\in\{0,1,\ldots,K\}:\\
&\qquad \sum_{i\le r}s_i(W_0)\le \rho
\Bigr\}.
\end{aligned}
\end{equation}

well defined since the $r=0$ sum is $0\le\rho$ and the constraint is restricted to $r\le K$ (for $r>K$ the partial sums are constant, so no information is lost by the restriction). Defining the budget-aware certified rank
\begin{equation}
\begin{aligned}
r_{\mathrm{cert}}(\epsilon,\rho)
&:= \min\Bigl\{
1\le r\le K:\ \tau_r\le \epsilon/\kappa,\\
&\qquad
\sum_{i\le r}s_i(W_0)\le \rho
\Bigr\},
\end{aligned}
\end{equation}

whenever the set is nonempty, we therefore get: the set is nonempty \emph{exactly} when \eqref{eq:budget-at-rspec} holds, and in that case $r_{\mathrm{cert}}(\epsilon,\rho)=r_{\mathrm{spec}}(\epsilon)$. So the budget never changes \emph{which} rank is minimally certifying; it only decides \emph{whether} any rank certifies at all. All bounds of Theorem~\ref{thm:minrank} and Proposition~\ref{prop:lower} thus apply verbatim to $r_{\mathrm{cert}}$ on its domain of definition.
\end{remark}

\begin{remark}[Variants of the complexity term]
\label{rem:gen-Ar}
Three comments on the complexity term of Theorem~\ref{thm:gen}. (a) For $\mathcal D_r=\mathcal A_r$, Proposition~\ref{prop:matrix-collapse} in place of Lemma~\ref{lem:matrix-rad} gives the exact layer complexity $\rho\,\mathbb E\|S\|_2/n_S$ instead of $\rho\,\mathbb E\|S_r\|_F/n_S$; since $\|S\|_2\le\|S\|_F$ and $\mathbb E\|S\|_F\le X_{\max}\sqrt{mn_S}$ by \eqref{eq:frob-branch}, the worst-case simplification is again $2\sqrt2L_\psi\rho X_{\max}\sqrt{m/n_S}$---\emph{identical} to the $\mathcal B_{r,\rho}$ worst case. The two parameterizations are distinguished only by their data-dependent complexities, $\rho\,\mathbb E\|S\|_2/n_S$ versus $\rho\,\mathbb E\|S_r\|_F/n_S$. (b) The data-dependent refinement $\tfrac{2\sqrt2L_\psi\rho}{n_S}\min\{\sqrt{q_r}\,\mathbb E_\sigma\|S\|_2,\mathbb E_\sigma\|S\|_F\}$ of Corollary~\ref{cor:rank-dep} is conditional on the sample; it can be used in Theorem~\ref{thm:gen} at the standard cost of stating the bound with the empirical Rademacher complexity and an additional deviation term. (c) In all variants the frozen weight is absent from the vector-valued complexity and from the contraction upper bound used here (Lemma~\ref{lem:translation}), and present in the transport term.
\end{remark}

\begin{remark}[Where $W_0$ does and does not appear]
\label{rem:w0-roles}
The frozen weight plays two distinct roles. It contributes nothing to the \emph{vector-valued layer complexity}---Lemma~\ref{lem:translation} is an exact identity, not a bound---and hence nothing to the contraction upper bound used in Theorem~\ref{thm:gen}, because it is never optimized against the sample. (We note the precise scope: the \emph{exact} scalar complexity of a fixed nonlinear readout class may still depend on $W_0$; the $W_0$-freeness is exact at the vector-valued level and inherited by the particular upper bound we use.) But it necessarily enters the \emph{transport} term through the input-space Lipschitz constant $L_\psi(\|W_0\|_2+\rho)$, because $h$ is $L_\psi$-Lipschitz in the layer \emph{output} $Wx$, not in the feature $x$, and different members of the class carry different $\Delta W$, so the transport step cannot be relocated to a shared post-layer feature space. Freezing $W_0$ removes its statistical cost, not its geometric influence.
\end{remark}

\begin{remark}[Two existential witnesses]
\label{rem:two-witnesses}
Membership in $\mathcal C_{r,\mathrm{two}}^{\mathrm{push}}(\epsilon)$ involves two existential quantifiers: one update realizes the target as a pushforward, and another (via Definition~\ref{def:cone}) certifies the discrepancy condition. The definition does \emph{not} require the two witnesses to coincide, and our tightness construction (Proposition~\ref{prop:tight}) indeed uses two different rank-one updates---$\Delta W=\rho uu^T$ to transport the source, $\Delta W'=-W_0uu^T$ to certify transferability. Readers interested in the single-update variant
\begin{equation}
\begin{aligned}
\mathcal C_{r,\mathrm{same}}^{\mathrm{push}}(\epsilon)
&:= \Bigl\{(I+\Delta W)_\#D_S:\ \Delta W\in\mathcal A_r,\\
&\qquad
\tilde d_{\mathcal H_{W_0+\Delta W}}
\bigl(D_S,(I+\Delta W)_\#D_S\bigr)
\le \epsilon\Bigr\}
\end{aligned}
\end{equation}
should note that $\mathcal C_{r,\mathrm{same}}^{\mathrm{push}}(\epsilon)\subseteq\mathcal C_{r,\mathrm{two}}^{\mathrm{push}}(\epsilon)$, so the upper bound of Theorem~\ref{thm:cone}(i) applies to it verbatim; whether that upper bound is attained for the single-update cone in general is an open question that our witness does not settle.
\end{remark}

\begin{remark}[The collapses are one phenomenon]
\label{rem:same-phenomenon}
Theorem~\ref{thm:cone}(i) and Propositions~\ref{prop:collapse}--\ref{prop:matrix-collapse} are two readings of Lemma~\ref{lem:extremal}: over a nuclear ball, both the Rademacher functional \eqref{eq:extremal-lin} and the Frobenius norm \eqref{eq:extremal-frob} are maximized at rank one, so the rank cap never binds. Recovering rank-dependence therefore requires abandoning the per-factor budget framing---statistically, by passing to the \emph{matrix-valued} class under a joint rank-and-Frobenius constraint (Lemma~\ref{lem:matrix-rad}), the scalar linear formalizations remaining rank-free even after both modifications (Proposition~\ref{prop:necessity}); geometrically, by asking not what maximal radius a fixed budget attains but how large a budget is \emph{required} to reach a specified spectral target (Theorem~\ref{thm:minrank}, Proposition~\ref{prop:lower}, Theorem~\ref{thm:cone}(ii)). We regard this as a structural finding about Frobenius-budgeted LoRA, not a limitation of technique: under the budget condition of Proposition~\ref{prop:tight} the reachable radius is exactly rank-independent for the witness instance, not merely rank-independently bounded.
\end{remark}

\section{Proofs of Main Results}
\label{app:proofs}

\subsection{Proof of Lemma~\ref{lem:factorization}}

\begin{proof}
($\subseteq$) If $\Delta W=BA$ with $\|B\|_F\le B_B$, $\|A\|_F\le B_A$, then $\mathrm{rank}(\Delta W)\le r$ and, by Lemma~\ref{lem:submult}, $\|\Delta W\|_*\le\|B\|_F\|A\|_F\le\rho$.

($\supseteq$) Let $\mathrm{rank}(\Delta W)=k\le r$ and $\|\Delta W\|_*\le\rho$. Write the compact SVD $\Delta W=U\Lambda V^T$ with $\Lambda\in\mathbb R^{k\times k}$ diagonal and positive, and pad $U$, $\Lambda$, $V^T$ with zeros to sizes $m\times r$, $r\times r$, $r\times d$ respectively, so that $B_0:=U\Lambda^{1/2}\in\mathbb R^{m\times r}$ and $A_0:=\Lambda^{1/2}V^T\in\mathbb R^{r\times d}$. Then $B_0A_0=\Delta W$ and
\begin{equation}
\|B_0\|_F^2=\|A_0\|_F^2=\mathrm{tr}(\Lambda)=\|\Delta W\|_*,
\end{equation}
so $\|B_0\|_F\|A_0\|_F=\|\Delta W\|_*\le B_BB_A$. Rescaling $B:=cB_0$, $A:=A_0/c$ preserves the product, and both per-factor constraints hold for any
\begin{equation}
c\in\big[\|A_0\|_F/B_A,\ B_B/\|B_0\|_F\big],
\end{equation}
an interval that is nonempty precisely because $\|A_0\|_F\|B_0\|_F\le B_AB_B$. (If $\Delta W=0$ take $B=A=0$.)
\end{proof}

\subsection{Proof of Lemma~\ref{lem:extremal}}

\begin{proof}
\eqref{eq:extremal-lin}: For $\Delta W\in\mathcal A_r$, Hölder duality of $\|\cdot\|_*$ and $\|\cdot\|_2$ with Lemma~\ref{lem:factorization} gives $\langle\Delta W,M\rangle_F\le\|\Delta W\|_*\|M\|_2\le\rho\|M\|_2$. If $M=0$ both sides of \eqref{eq:extremal-lin} vanish and $\Delta W=0$ attains the supremum; assume then $M\ne0$. Conversely, let $M=\sum_is_i(M)\,p_iq_i^T$ be an SVD and put $\Delta W^\star:=\rho\,p_1q_1^T$. Then $\mathrm{rank}(\Delta W^\star)=1\le r$ and $\|\Delta W^\star\|_*=\rho$, so $\Delta W^\star\in\mathcal A_1\subseteq\mathcal A_r$ by Lemma~\ref{lem:factorization}, and $\langle\Delta W^\star,M\rangle_F=\rho\,s_1(M)=\rho\|M\|_2$.

\eqref{eq:extremal-frob}: $\|\Delta W\|_F\le\|\Delta W\|_*\le\rho$ for every $\Delta W\in\mathcal A_r$, and the rank-one point $\rho\,pq^T$ (any unit $p,q$) attains $\|\cdot\|_F=\rho$.
\end{proof}

\subsection{Proof of Theorem~\ref{thm:da}}

\begin{proof}
Fix $h'\in\mathcal H$. By the triangle inequality pointwise under $\mathbb E_{D_T}$,
\begin{equation}
\begin{aligned}
\epsilon_T(h)
&= \mathbb{E}_T\bigl|h-f_T\bigr| \\
&\le \mathbb{E}_T\bigl|h'-f_T\bigr|
 + \mathbb{E}_T\bigl|h-h'\bigr| \\
&= \epsilon_T(h')
 + \mathbb{E}_T\bigl|h-h'\bigr|.
\end{aligned}
\end{equation}

By Definition~\ref{def:disc}, since $h,h'\in\mathcal H$,
\begin{equation}
\mathbb E_T|h-h'|\le\mathbb E_S|h-h'|+\tilde d_{\mathcal H}(D_S,D_T),
\end{equation}
and by the triangle inequality again, through $f_S$ under $\mathbb E_{D_S}$,
\begin{equation}
\mathbb E_S|h-h'|\le\mathbb E_S|h-f_S|+\mathbb E_S|h'-f_S|=\epsilon_S(h)+\epsilon_S(h').
\end{equation}
Chaining and taking the infimum over $h'$ gives the claim.
\end{proof}

\subsection{Proof of Proposition~\ref{prop:collapse}}

\begin{proof}
(i) $\mathcal F_1\subseteq\mathcal F_r$ is immediate, since $\mathcal A_1\subseteq\mathcal A_r$. Conversely take $h_{\mathbf w,\Delta W}\in\mathcal F_r$; if $\mathbf w=0$ the function is $0\in\mathcal F_1$, so assume $\mathbf w\ne0$ and define the rank-one surrogate
\begin{equation}
\Delta W' := \frac{1}{\|\mathbf w\|_2^{2}}\,\mathbf w\,\big(\mathbf w^T\Delta W\big).
\end{equation}
The normalization is what makes the readout unchanged: $\mathbf w^T\Delta W'=\|\mathbf w\|_2^{-2}\|\mathbf w\|_2^2\,\mathbf w^T\Delta W=\mathbf w^T\Delta W$, hence $h_{\mathbf w,\Delta W'}=h_{\mathbf w,\Delta W}$ pointwise, with the \emph{same} $\mathbf w$ and therefore the same frozen term $\mathbf w^TW_0x$. (One cannot instead rescale $\mathbf w$ to unit norm: that would alter $\mathbf w^TW_0x$.) Admissibility: $\Delta W'$ has rank one, and since for rank-one matrices $\|\cdot\|_*=\|\cdot\|_F$,
\begin{equation}
\begin{aligned}
\|\Delta W'\|_*
&= \frac{\|\mathbf w\|_2\,
        \|\Delta W^\top\mathbf w\|_2}
       {\|\mathbf w\|_2^2} \\
&= \frac{\|\Delta W^\top\mathbf w\|_2}
        {\|\mathbf w\|_2} \\
&\le \|\Delta W\|_2
 \le \|\Delta W\|_*
 \le \rho .
\end{aligned}
\end{equation}

so $\Delta W'\in\mathcal A_1$ by Lemma~\ref{lem:factorization}. Hence $\mathcal F_r\subseteq\mathcal F_1$.

(ii) With $v=\sum_i\sigma_ix_i$ we have $\sum_i\sigma_ih_{\mathbf w,\Delta W}(x_i)=\mathbf w^T(W_0+\Delta W)v$. For fixed $\mathbf w$, Lemma~\ref{lem:extremal} applied to $M=\mathbf wv^T$ gives
\begin{equation}
\begin{aligned}
\sup_{\Delta W\in\mathcal A_r}
 \mathbf w^\top\Delta Wv
&= \sup_{\Delta W\in\mathcal A_r}
 \langle \Delta W,\mathbf wv^\top\rangle_F \\
&= \rho\|\mathbf wv^\top\|_2 \\
&= \rho\|\mathbf w\|_2\|v\|_2 .
\end{aligned}
\end{equation}

Writing $\mathbf w=s\hat{\mathbf w}$ with $s\in[0,1]$, $\|\hat{\mathbf w}\|_2=1$, the remaining supremum is
\begin{equation}
\sup_{s\in[0,1]}\ s\Big(\sup_{\|\hat{\mathbf w}\|_2=1}\hat{\mathbf w}^TW_0v+\rho\|v\|_2\Big)=\|W_0v\|_2+\rho\|v\|_2,
\end{equation}
the bracket being nonnegative so that $s=1$ is optimal. Dividing by $n$ and taking $\mathbb E_\sigma$ gives the identity, which visibly contains no $r$. Finally $\mathbb E_\sigma\|v\|_2\le(\mathbb E_\sigma\|v\|_2^2)^{1/2}=(\sum_i\|x_i\|_2^2)^{1/2}\le\sqrt nX_{\max}$ by Jensen and independence, and $\|W_0v\|_2\le\|W_0\|_2\|v\|_2$, giving \eqref{eq:collapsed}. That no covering-number or chaining refinement can introduce $r$-dependence is immediate from (i), since all such functionals depend on the class only as a set of functions.
\end{proof}

\subsection{Proof of Proposition~\ref{prop:matrix-collapse}}

\begin{proof}
For fixed $\{\sigma_i\}$, $\sum_i\langle\sigma_i,\Delta Wx_i\rangle=\langle\Delta W,\sum_i\sigma_ix_i^T\rangle_F=\langle\Delta W,S\rangle_F$. By Lemma~\ref{lem:extremal} with $M=S$, the supremum over $\mathcal A_r$ equals $\rho\|S\|_2$ for every $r\ge1$. Divide by $n$ and take $\mathbb E_\sigma$.
\end{proof}

\subsection{Proof of Proposition~\ref{prop:necessity}}

\begin{proof}
(i) For fixed $\mathbf w$, $\sum_i\sigma_i\mathbf w^T\Delta Wx_i=\langle\Delta W,\mathbf wv^T\rangle_F$, and Lemma~\ref{lem:extremal} gives supremum $\rho\|\mathbf wv^T\|_2=\rho\|\mathbf w\|_2\|v\|_2$, with no $r$.

(ii) For fixed $\mathbf w$, the signal matrix $M=\mathbf wv^T$ has rank one, so $M_r=M$ for every $r\ge1$. By Step 1 of the proof of Lemma~\ref{lem:matrix-rad} below (which is valid for any $M$),
\begin{equation}
\sup_{\Delta W\in\mathcal B_{r,\rho}}\langle\Delta W,M\rangle_F=\rho\|M_r\|_F=\rho\|M\|_F=\rho\|\mathbf w\|_2\|v\|_2 .
\end{equation}
Taking the supremum over $\|\mathbf w\|_2\le1$ and dividing by $n$ gives $\tfrac\rho n\mathbb E\|v\|_2$, with no $r$.

(iii) Identical to (ii) with the supremum over $\mathbf w$ omitted: the signal matrix $\mathbf wv^T$ is rank one, so $(\mathbf wv^T)_r=\mathbf wv^T$ and the supremum over $\mathcal B_{r,\rho}$ equals $\rho\|\mathbf wv^T\|_F=\rho\|\mathbf w\|_2\|v\|_2$ for every $r\ge1$.
\end{proof}

\subsection{Proof of Lemma~\ref{lem:matrix-rad}}

\begin{proof}
Throughout, $s_i(\cdot)$ denotes singular values, kept notationally distinct from the Rademacher vectors $\sigma_i$.

\emph{Step 1 (exact value of the supremum).} Let $\Delta W$ have singular values $t_1\ge t_2\ge\cdots$, with $t_i=0$ for $i>r$. By von Neumann's trace inequality and then Cauchy--Schwarz with $\sum_it_i^2=\|\Delta W\|_F^2\le\rho^2$,
\begin{equation}
\begin{aligned}
\langle \Delta W,S\rangle_F
&\le \sum_{i\ge 1} t_i\,s_i(S)
 = \sum_{i=1}^r t_i\,s_i(S) \\
&\le \rho
 \left(\sum_{i=1}^r s_i(S)^2\right)^{1/2} \\
&= \rho\|S_r\|_F .
\end{aligned}
\end{equation}

If $S_r=0$ (equivalently $S=0$) the supremum is zero and is attained at $\Delta W=0$; otherwise the value is attained at $\Delta W^\star=\rho\,S_r/\|S_r\|_F$ (rank $\le r$, Frobenius norm $\rho$), since $\langle S_r,S\rangle_F=\|S_r\|_F^2$. In either case $\sup_{\mathcal B_{r,\rho}}\langle\Delta W,S\rangle_F=\rho\|S_r\|_F$ exactly.

\emph{Step 2 (two bounds on $\|S_r\|_F$).} $S_r$ retains only the top $r$ singular values, of which at most $q_r=\min\{r,m,d\}$ are nonzero, so
\begin{equation}
\|S_r\|_F\le\min\big(\sqrt{q_r}\,s_1(S),\ \|S\|_F\big)=\min\big(\sqrt{q_r}\,\|S\|_2,\ \|S\|_F\big).
\end{equation}
The first branch uses the rank constraint; the second holds for any matrix and gives the rank-free baseline.

\emph{Step 3 (Frobenius branch).} By independence and $\mathbb E\sigma_i=0$ the cross terms vanish, and $\|\sigma_ix_i^T\|_F^2=\|\sigma_i\|_2^2\|x_i\|_2^2=m\|x_i\|_2^2$ deterministically, so by Jensen
\begin{equation}
\mathbb E\|S\|_F\le\big(\mathbb E\|S\|_F^2\big)^{1/2}=\Big(\sum_im\|x_i\|_2^2\Big)^{1/2}=\sqrt{m\,\mathrm{tr}(G_n)},
\end{equation}
which is \eqref{eq:frob-branch}.

\emph{Step 4 (operator branch via matrix Bernstein).} Let $Z_i=\sigma_ix_i^T$, independent and mean zero. Since $\mathbb E[\sigma_i\sigma_i^T]=I_m$,
\begin{equation}
\Big\|\sum_i\mathbb E[Z_iZ_i^T]\Big\|_2=\Big\|\sum_i\|x_i\|_2^2I_m\Big\|_2=\mathrm{tr}(G_n),
\end{equation}
and since $\sigma_i^T\sigma_i=m$ deterministically,
\begin{equation}
\Big\|\sum_i\mathbb E[Z_i^TZ_i]\Big\|_2=m\Big\|\sum_ix_ix_i^T\Big\|_2=m\|G_n\|_2 .
\end{equation}
Hence the variance proxy is at most $v_n=\max(\mathrm{tr}(G_n),m\|G_n\|_2)$, and $\|Z_i\|_2=\|\sigma_i\|_2\|x_i\|_2\le\sqrt mX_{\max}=:R$ almost surely. The expectation form of the matrix Bernstein inequality \cite{tropp2012user} gives \eqref{eq:op-branch}. Combining Steps 1--4 and dividing by $n$ proves (i).

\emph{Step 5 (worst case).} Under Assumption~\ref{ass:bddfeat}, $\mathrm{tr}(G_n)\le nX_{\max}^2$ and $\|G_n\|_2\le nX_{\max}^2$, so $v_n\le mnX_{\max}^2$; and when $\log(m+d)\le n$,
\begin{equation}
\tfrac{\sqrt mX_{\max}}3\log(m+d)\le\tfrac{X_{\max}}3\sqrt{mn\log(m+d)} .
\end{equation}
Therefore $\mathbb E\|S\|_2\le(\sqrt2+\tfrac13)X_{\max}\sqrt{mn\log(m+d)}\le2X_{\max}\sqrt{mn\log(m+d)}$, while $\mathbb E\|S\|_F\le X_{\max}\sqrt{mn}$. Substituting into (i) and dividing by $n$ yields the minimum in \eqref{eq:lem2-fixed}. For the final equality in \eqref{eq:lem2-fixed}, compare the two branches: $2\sqrt{mq_r\log(m+d)/n}\ge\sqrt{m/n}$ if and only if $4q_r\log(m+d)\ge1$, which holds for every $q_r\ge1$ and $m+d\ge2$ since $4\log2>1$; the minimum is therefore always the Frobenius branch.
\end{proof}

\subsection{Proof of Corollary~\ref{cor:rank-dep}}

\begin{proof}
Apply Lemma~\ref{lem:contraction} with the appropriate constraint set. For $\mathcal D_r=\mathcal A_r$, $\mathfrak R_n^{\mathrm{mat}}(\mathcal A_r)=\tfrac\rho n\mathbb E\|S\|_2$ by Proposition~\ref{prop:matrix-collapse}; for $\mathcal D_r=\mathcal B_{r,\rho}$, Lemma~\ref{lem:matrix-rad}(i) gives the data-dependent minimum. In both lines the final inequality uses $\|S\|_2\le\|S\|_F$ (resp.\ $\|S_r\|_F\le\|S\|_F$) and $\mathbb E\|S\|_F\le X_{\max}\sqrt{mn}$ from \eqref{eq:frob-branch}. The worst-case bounds are uniform over samples obeying Assumption~\ref{ass:bddfeat}, hence hold for the expected complexity as well.
\end{proof}

\subsection{Proof of Lemma~\ref{lem:monotone}}

\begin{proof}
If $W_0=0$ then $\tau_r=0$ for every $r$ and the claim is immediate; likewise, for $r\ge\mathrm{rank}(W_0)$ both $\tau_r$ and $\tau_{r+1}$ vanish and there is nothing to prove. So assume $W_0\ne0$ and $r+1\le\mathrm{rank}(W_0)$. Write $M_r:=(W_0)_{>r}\Delta\Sigma(W_0)_{>r}^T$, so $\tau_r=\|M_r\|_2$. Let $u_{r+1}$ be the $(r{+}1)$-st left singular vector of $W_0$ and $P:=I_m-u_{r+1}u_{r+1}^T$. Since $(W_0)_{>r+1}=(W_0)_{>r}-s_{r+1}(W_0)u_{r+1}v_{r+1}^T=P\,(W_0)_{>r}$, we get $M_{r+1}=PM_rP$, whence $\tau_{r+1}=\|PM_rP\|_2\le\|P\|_2^2\|M_r\|_2=\tau_r$. If $r\ge\mathrm{rank}(W_0)$ then $(W_0)_{>r}=0$ and $\tau_r=0$, so the defining set of $r_{\mathrm{spec}}(\epsilon)$ is nonempty and, by monotonicity, upward closed.
\end{proof}

\subsection{Proof of Theorem~\ref{thm:minrank}}

\begin{proof}
If $\|\Delta\Sigma\|_2=0$ then $\tau_r=0$ for every $r$, so $r_{\mathrm{spec}}(\epsilon)=1$; under the stated budget condition, Lemma~\ref{lem:spectral} makes the cancellation admissible with alignment value $\tau_{r_{\mathrm{spec}}(\epsilon)}=0$, and Assumption~\ref{ass:kappa} then gives $\tilde d_{\mathcal H_W}(D_S,D_T)\le\kappa\,\tau_{r_{\mathrm{spec}}(\epsilon)}=0\le\epsilon$, proving the certificate claim; the decay bounds are not needed in this branch and are not evaluated. Assume henceforth $\|\Delta\Sigma\|_2>0$. Well-definedness of $r_{\mathrm{spec}}(\epsilon)$ is Lemma~\ref{lem:monotone} together with $\tau_K=0$. Under \eqref{eq:budget-at-rspec}, Lemma~\ref{lem:spectral} makes the top-block cancellation at rank $r_{\mathrm{spec}}(\epsilon)$ admissible with $\|W\Delta\Sigma W^T\|_2=\tau_{r_{\mathrm{spec}}(\epsilon)}\le\epsilon/\kappa$, and Assumption~\ref{ass:kappa} (applied at the cancellation weight $W=(W_0)_{>r_{\mathrm{spec}}(\epsilon)}$) gives $\tilde d_{\mathcal H_W}(D_S,D_T)\le\kappa\tau_{r_{\mathrm{spec}}(\epsilon)}\le\epsilon$.

For the decay bounds: by submultiplicativity of the operator norm, $\tau_r\le\|(W_0)_{>r}\|_2^2\|\Delta\Sigma\|_2=s_{r+1}(W_0)^2\|\Delta\Sigma\|_2$. Under polynomial decay and for $r\ge1$,
\begin{equation}
\tau_r\le\bar C^2(r+1)^{-2\alpha}\|\Delta\Sigma\|_2\le\bar C^2r^{-2\alpha}\|\Delta\Sigma\|_2,
\end{equation}
so $\tau_r\le\epsilon/\kappa$ holds for every integer $r\ge(\kappa\bar C^2\|\Delta\Sigma\|_2/\epsilon)^{1/(2\alpha)}$, in particular for the integer $\lceil(\kappa\bar C^2\|\Delta\Sigma\|_2/\epsilon)^{1/(2\alpha)}\rceil$ (which is at least $1$, the argument of the ceiling being positive); by Lemma~\ref{lem:monotone} the set $\{r:\tau_r\le\epsilon/\kappa\}$ is upward closed, so $r_{\mathrm{spec}}(\epsilon)$ is at most that integer. In the geometric case, $\bar C^2\beta^{2r}\|\Delta\Sigma\|_2\le\epsilon/\kappa$ holds for every integer $r\ge\tfrac1{2\log(1/\beta)}\log(\kappa\bar C^2\|\Delta\Sigma\|_2/\epsilon)$; if the logarithm is nonpositive this includes $r=1$, and otherwise it includes the ceiling of the displayed quantity, whence the stated bound with the outer $\max\{1,\cdot\}$. In both cases $r_{\mathrm{spec}}(\epsilon)\le K$ holds by definition, justifying the outer minimum with $K$.
\end{proof}

\subsection{Proof of Proposition~\ref{prop:lower}}

\begin{proof}
Take the unit vector $w=u_{r+1}$. Since $u_{r+1}^T(W_0)_{>r}=s_{r+1}(W_0)\,v_{r+1}^T$,
\begin{equation}
\begin{aligned}
\tau_r
&\ge \bigl|
w^\top (W_0)_{>r}\Delta\Sigma
(W_0)_{>r}^\top w
\bigr| \\
&= s_{r+1}(W_0)^2
 \bigl|v_{r+1}^\top\Delta\Sigma v_{r+1}\bigr| \\
&\ge c_1s_{r+1}(W_0)^2 .
\end{aligned}
\end{equation}

Under the two-sided decay bounds, $\tau_r\ge c_1c_0^2(r+1)^{-2\alpha}$ (resp.\ $c_1c_0^2\beta^{2(r+1)}$). In the non-saturated regime the index $r_{\mathrm{spec}}(\epsilon)+1\le K$ is available, so $v_{r_{\mathrm{spec}}(\epsilon)+1}$ exists, $s_{r_{\mathrm{spec}}(\epsilon)+1}(W_0)>0$, and the uniform directional condition covers it; the displayed lower bound therefore applies at $r=r_{\mathrm{spec}}(\epsilon)$. Since $\tau_{r_{\mathrm{spec}}(\epsilon)}\le\epsilon/\kappa$ by definition, this forces $(r_{\mathrm{spec}}(\epsilon)+1)^{2\alpha}\ge\kappa c_1c_0^2/\epsilon$ (resp.\ $2(r_{\mathrm{spec}}(\epsilon)+1)\log(1/\beta)\ge\log(\kappa c_1c_0^2/\epsilon)$), giving the stated bounds.
\end{proof}

\subsection{Proof of Theorem~\ref{thm:gen}}

\begin{proof}
\emph{Step 1 (adaptation).} Theorem~\ref{thm:da} with $\mathcal H=\mathcal H_r^\psi$ gives $\epsilon_T(h)\le\epsilon_S(h)+\tilde d_{\mathcal H_r^\psi}(D_S,D_T)+\lambda^*$.

\emph{Step 2 (transport, with the correct Lipschitz constant).} Every $h\in\mathcal H_r^\psi$ satisfies, for all $x,x'$,
\begin{equation}
|h(x)-h(x')|\le L_\psi\|(W_0+\Delta W)(x-x')\|_2\le L_\psi M\|x-x'\|_2,
\end{equation}
since $\|W_0+\Delta W\|_2\le\|W_0\|_2+\|\Delta W\|_F\le M$. Hence for any $h,h'\in\mathcal H_r^\psi$ the function $\varphi:=|h-h'|$ is $2L_\psi M$-Lipschitz, being the composition of the $1$-Lipschitz $|\cdot|$ with a difference of two $L_\psi M$-Lipschitz functions. If $L_\psi M=0$, every member of the class is constant on $\mathcal X$, so $\tilde d_{\mathcal H_r^\psi}(D_S,D_T)=0$ and the transport term may be taken to be zero; assume therefore $L_\psi M>0$. Kantorovich--Rubinstein duality applied to $\varphi/(2L_\psi M)$---licit since $D_S$ is compactly supported under Assumption~\ref{ass:bddfeat} and $D_T$ has a finite first moment by hypothesis---yields
\begin{equation}
\tilde d_{\mathcal H_r^\psi}(D_S,D_T)\le2L_\psi M\,W_1(D_S,D_T),
\end{equation}
following the optimal-transport route of \cite{redko2017theoretical, shen2018wasserstein}. The factor $M$ cannot be dropped, for the reason given in Remark~\ref{rem:w0-roles}.

\emph{Step 3 (concentration).} The source loss class $\mathcal L=\{x\mapsto|h(x)-f_S(x)|:h\in\mathcal H_r^\psi\}$ takes values in $[0,b]$ by Assumption~\ref{ass:bounded}, so symmetrization plus McDiarmid gives, with probability $\ge1-\delta$, uniformly in $h$,
\begin{equation}
\epsilon_S(h)\le\hat\epsilon_S(h)+2\,\bar{\mathfrak R}_{n_S}(\mathcal L)+b\sqrt{\tfrac{\log(1/\delta)}{2n_S}},
\end{equation}
where $\bar{\mathfrak R}_{n_S}$ is the expected complexity of Section~\ref{sec:framework}, as is standard for this inequality; since our complexity bounds below are uniform over samples satisfying Assumption~\ref{ass:bddfeat}, they bound the expected complexity as well. To contract, note that $\phi_i(t):=|t-f_S(x_i)|$ is $1$-Lipschitz but does \emph{not} vanish at $0$; the version of the Ledoux--Talagrand contraction principle stated with absolute values inside the supremum \cite{ledoux1991probability} requires the contractions to vanish at $0$, so to keep the argument self-contained we center. Set $\tilde\phi_i(t):=|t-f_S(x_i)|-|f_S(x_i)|$, which is $1$-Lipschitz with $\tilde\phi_i(0)=0$. Then for every $h$,
\begin{equation}
\sum_i\sigma_i\phi_i(h(x_i))=\sum_i\sigma_i\tilde\phi_i(h(x_i))+\sum_i\sigma_i|f_S(x_i)|,
\end{equation}
and the last sum does not depend on $h$ and has zero $\sigma$-expectation. Hence, samplewise, $\mathfrak R_{n_S}(\mathcal L)=\mathfrak R_{n_S}(\{\tilde\phi_i\circ h\})\le\mathfrak R_{n_S}(\mathcal H_r^\psi)$ by contraction; taking expectations over the sample, the same holds for $\bar{\mathfrak R}_{n_S}$, and Corollary~\ref{cor:rank-dep} bounds the latter by $\sqrt2L_\psi\rho X_{\max}\sqrt{m/n_S}$ for either constraint set $\mathcal D_r$; by Lemma~\ref{lem:translation} this term carries no $W_0$ contribution. Combining Steps 1--3 gives the statement.
\end{proof}

\subsection{Proof of Corollary~\ref{cor:spectral-route}}
\begin{proof}
The discrepancy bound is immediate from Lemma~\ref{lem:spectral} (admissibility and $\|W\Delta\Sigma W^T\|_2=\tau_r$) and Assumption~\ref{ass:kappa} applied at $W$; (i) is then Theorem~\ref{thm:da} with $\mathcal H=\mathcal H_W$. For (ii), bound $\epsilon_S(h)$ by symmetrization plus McDiarmid exactly as in Step 3 of the proof of Theorem~\ref{thm:gen}: the loss class of $\mathcal H_W$ takes values in $[0,b]$, and the centering argument there applies verbatim to the fixed-weight class, giving $\epsilon_S(h)\le\hat\epsilon_S(h)+2\bar{\mathfrak R}_{n_S}(\mathcal H_W)+b\sqrt{\log(1/\delta)/(2n_S)}$ uniformly in $h$. For the linear-readout instantiation, for any unit-ball readout set contained in $\{\mathbf w:\|\mathbf w\|_2\le1\}$,

\begin{equation}
\begin{aligned}
\mathbb E_\sigma \sup_{\mathbf w}
 \sum_i \sigma_i \mathbf w^\top W x_i
&\le \mathbb E_\sigma \|Wv\|_2 \\
&\le \|W\|_2\,\mathbb E_\sigma \|v\|_2 \\
&\le \|W\|_2\sqrt{n_S}\,X_{\max}.
\end{aligned}
\end{equation}

with $v=\sum_i\sigma_ix_i$ (for the range-adapted family the first inequality is an equality, since $Wv\in\mathrm{range}(W)$); divide by $n_S$ and use $\|(W_0)_{>r}\|_2=s_{r+1}(W_0)$, valid also at $r=K$ under the convention $s_i(W_0)=0$ for $i>K$. The bound is uniform over samples satisfying Assumption~\ref{ass:bddfeat}, hence bounds the expected complexity $\bar{\mathfrak R}_{n_S}$.
\end{proof}

\subsection{Proof of Proposition~\ref{prop:trivial-radius}}

\begin{proof}
Let $W=W_0+\Delta W$ with $\Delta W\in\mathcal A_r$ and pick a unit $k\in\ker(W)$. For any $D$ and any $t$, the translate $D^{(t)}:=(x\mapsto x+tk)_\#D$ satisfies $W(x+tk)=Wx$, so the law of $\psi(Wx)$ is the same under $D$ and $D^{(t)}$ for every readout $\psi$. Hence $\mathbb E_{D}|h-h'|=\mathbb E_{D^{(t)}}|h-h'|$ for all $h,h'\in\mathcal H_W$, giving $\tilde d_{\mathcal H_W}(D,D^{(t)})=0\le\epsilon$ and therefore $D^{(t)}\in\mathcal C_r(W_0,D;\epsilon)$.

(i) If $\mathcal X=\mathbb R^d$, take $D=D_S$ arbitrary with finite second moment. The map $x\mapsto x+tk$ is a coupling, and by the projection argument $W_2(D_S,D_S^{(t)})=t\|k\|_2=t$, which is unbounded in $t$.

(ii) If $\mathcal X$ is the ball of radius $X_{\max}$, choose $D_S$ supported in $\{x\in k^\perp:\|x\|_2\le X_{\max}/2\}$, which is a nonempty subset of $\mathcal X$. For $t\le X_{\max}/2$ the translate $D_S^{(t)}$ is supported in $\mathcal X$ by orthogonality, since $\|x+tk\|_2^2=\|x\|_2^2+t^2\le X_{\max}^2/2$. Taking $t=X_{\max}/2$ gives a cone member at $W_2$-distance $X_{\max}/2$. In both cases the value obtained is independent of $\epsilon$, and of $r$ and $\rho$ within the kernel premise.
\end{proof}

\subsection{Proof of Theorem~\ref{thm:cone}}

\begin{proof}
(i) For any $\Delta W\in\mathcal A_r$, the map $x\mapsto(I+\Delta W)x$ is one feasible coupling of $D_S$ with its pushforward, so
\begin{equation}
\begin{aligned}
W_2^2\bigl(D_S,(I+\Delta W)_\#D_S\bigr)
&\le \mathbb{E}_{x\sim D_S}
   \bigl\|\Delta W x\bigr\|_2^2 \\
&= \operatorname{Tr}
   \bigl(\Delta W\Sigma_S\Delta W^\top\bigr).
\end{aligned}
\end{equation}

and $\mathrm{Tr}(\Delta W\Sigma_S\Delta W^T)\le\lambda_{\max}(\Sigma_S)\|\Delta W\|_F^2$. By Lemma~\ref{lem:extremal}\eqref{eq:extremal-frob}, $\sup_{\Delta W\in\mathcal A_r}\|\Delta W\|_F=\rho$ with no factor of $\sqrt r$ and no dependence on $r$. Taking square roots and the supremum over $\mathcal C_{r,\mathrm{two}}^{\mathrm{push}}(\epsilon)\subseteq\{(I+\Delta W)_\#D_S:\Delta W\in\mathcal A_r\}$ gives the bound; sharpness under the budget condition $s_d(W_0)\le\rho$ is Proposition~\ref{prop:tight}, whose part (i) supplies the required cone membership and whose part (ii) supplies the extremal displacement. The identical argument applies to $\mathcal B_{r,\rho}$, over which $\sup\|\Delta W\|_F=\rho$ as well.

(ii) By Lemma~\ref{lem:spectral} the cancellation update is admissible under the stated nuclear budget and achieves $\|W\Delta\Sigma W^T\|_2=\tau_{\mathrm{tail}}$ with $W=(W_0)_{>r}$. Assumption~\ref{ass:kappa}, applied at this cancellation weight, then gives $\tilde d_{\mathcal H_W}(D_S,D_T)\le\kappa\tau_{\mathrm{tail}}\le\epsilon$, which is exactly the membership condition of Definition~\ref{def:cone}.
\end{proof}

\end{document}